%% file: nips_main.tex
\documentclass{article}

     \usepackage[final,nonatbib]{neurips_2020}

\usepackage[utf8]{inputenc} 
\usepackage[T1]{fontenc}    
\usepackage{hyperref}       
\usepackage{url}            
\usepackage{booktabs}       
\usepackage{nicefrac}       
\usepackage{microtype}      

\usepackage{mkolar_definitions}
\usepackage{graphicx}
\usepackage{algorithm}
\usepackage[noend]{algorithmic}
\usepackage{amsfonts,amssymb,amsmath,nicefrac,xspace}
\usepackage{xcolor}
\usepackage{enumitem}

\newtheorem{assumption}{Assumption}

\newcommand{\defeq}{\stackrel{\text{def}}{=}}

\renewcommand{\(}{\left(}
\renewcommand{\)}{\right)}
\renewcommand{\[}{\left[}
\renewcommand{\]}{\right]}

\DeclareMathOperator*{\Reg}{\mathrm{Reg}}
\DeclareMathOperator*{\avg}{avg}

\DeclareMathOperator*{\EX}{\mathrm{EX}}

\newcommand{\poly}[1]{\operatorname{poly}\(#1\)}
\newcommand{\polylog}[1]{\operatorname{polylog}\(#1\)}
\newcommand{\sign}[1]{\operatorname{sign}\(#1\)}

\newcommand{\R}{\mathbb{R}}
\newcommand{\Rd}{\mathbb{R}^d}

\newcommand{\calX}{\mathcal{X}}
\newcommand{\calY}{\mathcal{Y}}
\newcommand{\calH}{\mathcal{H}}
\newcommand{\calO}{\mathcal{O}}
\newcommand{\calF}{\mathcal{F}}
\newcommand{\err}{\mathrm{err}}

\newcommand{\zeronorm}[1]{\left\lVert #1 \right\rVert_{0}}
\newcommand{\twonorm}[1]{\left\lVert #1 \right\rVert}

\newcommand{\refine}{\textsc{Refine}\xspace}
\newcommand{\init}{\textsc{Initialize}\xspace}

\usepackage{pifont}
\newcommand{\cmark}{\text{\ding{51}}}%
\newcommand{\xmark}{\text{\ding{55}}}

\newcommand{\optu}{\tilde{u}}

\newcommand{\awasthifull}{\tilde{\order}\big(d^{2^{\poly{1/(1-2\eta)}}} \cdot \ln\frac{1}{\epsilon}\big)}

\newcommand{\awasthisparse}{\tilde{\order}\big((\frac{s \ln d}{\epsilon})^{2^{\poly{1/(1-2\eta)}}} \big)}

\title{Efficient active learning of sparse halfspaces with arbitrary bounded noise}

\author{
Chicheng Zhang\\
  University of Arizona\\
  \texttt{chichengz@cs.arizona.edu} \\
  \And
  Jie Shen\\
Stevens Institute of Technology\\
\texttt{jie.shen@stevens.edu}\\
\And
Pranjal Awasthi\\
Google Research and Rutgers University\\
\texttt{pranjalawasthi@google.com}
}

\begin{document}
\maketitle

\input{abstract}
\input{introduction}

\input{prelims}

\input{algorithm}

\input{guarantees}

\input{conclusion}

\section*{Broader Impact}

This paper investigates a fundamental problem in machine learning and statistics. The theory and algorithms presented in this paper are expected to benefit many broad fields in science and engineering, such as learning theory, robust statistics, optimization, and applications in biology, climatology, and seismology, to name a few.  Our research belongs to the general paradigm of interactive learning, in which the learning agent need to design adaptive sampling schemes to maximize data efficiency. We are well aware that one needs to be careful in designing such sampling schemes, to avoid unintended harms such as discrimination.

\begin{ack}
The authors would like to thank Ning Hao and Hao Helen Zhang for helpful discussions on marginal screening for variable selection, which inspired the averaging-based initialization procedure in this paper. 
Chicheng Zhang acknowledges startup funding support from the University of Arizona.
Jie Shen is supported by NSF-IIS-1948133 and the startup funding of Stevens Institute of Technology.
\end{ack}

\bibliographystyle{plain}
\bibliography{ref}

\clearpage
\appendix

\input{related_work}

\input{appendix}

\end{document}

%% file: abstract.tex
\begin{abstract}\label{sec:abstract}
We study active learning of homogeneous $s$-sparse halfspaces in $\mathbb{R}^d$ under the setting where the unlabeled data distribution is isotropic log-concave and each label is flipped with probability at most $\eta$ for a parameter $\eta \in \big[0, \frac12\big)$, known as the bounded noise. 
Even in the presence of mild label noise, i.e. $\eta$ is a small constant, this is a challenging problem and only recently have label complexity bounds of the form $\tilde{\order}\(s \cdot \polylog{d, \frac{1}{\epsilon}}\)$ been established in [Zhang, 2018] for computationally efficient algorithms. In contrast, under high levels of label noise, the label complexity bounds achieved by computationally efficient algorithms are much worse:~the best known result of~[Awasthi et al., 2016] provides a computationally efficient algorithm with label complexity $\awasthisparse$, which is label-efficient only when the noise rate $\eta$ is a fixed constant. In this work, we substantially improve on it by designing a polynomial time algorithm for active learning of $s$-sparse halfspaces, with a label complexity of $\tilde{\order}\big(\frac{s}{(1-2\eta)^4} \polylog {d, \frac 1 \epsilon} \big)$. This is the first efficient algorithm with label complexity polynomial in $\frac{1}{1-2\eta}$ in this setting, which is label-efficient even for $\eta$ arbitrarily close to $\frac12$. Our active learning algorithm and its theoretical guarantees also immediately translate to new state-of-the-art label and sample complexity results for full-dimensional active and passive halfspace learning under arbitrary bounded noise. The key insight of our algorithm and analysis is a new interpretation of online learning regret inequalities, which may be of independent interest.
\end{abstract}

%% file: introduction.tex
\section{Introduction}

In machine learning and statistics, linear classifiers (i.e. halfspaces) are arguably one of the most important models as witnessed by a long-standing research effort dedicated to establishing computationally efficient and provable algorithms for halfspace learning~\cite{rosenblatt1958perceptron,vapnik1998statistical,cristianini2000introduction}. In practical applications, however, data are often corrupted by various types of noise~\cite{sloan1988types,blum1996polynomial}, are expensive to annotate~\cite{cohn1994improving,dasgupta2009percerptron}, and are of high or even infinite dimensions~\cite{blum1990learning,candes2005decoding}. These characteristics rooted in contemporary machine learning problems pose new challenges to the design and analysis of learning algorithms for halfspaces. As a result, there has been extensive study of {\em noise-tolerant}, {\em label-efficient}, and {\em attribute-efficient} algorithms in the last few decades.

{\bfseries Noise-tolerant learning.} \
In the noiseless setting where there is a halfspace that has zero error rate with respect to the data distribution, it is well known that by simply finding a halfspace that fits all the training examples using linear programming, one is guaranteed vanishing generalization error~\cite{vapnik2015uniform}. In the presence of data corrpution, the success of efficient learning of halfspaces crucially depends on the underlying noise model. For instance, \cite{blum1996polynomial} proposed a polynomial time algorithm that provably learns halfspaces when the labels are corrupted by random classification noise, that is, each label is flipped independently with a fixed probability $\eta \in \big[0, \frac12\big)$. The bounded noise model, also known as Massart noise~\cite{sloan1988types,sloan1992Corrigendum,massart2006risk}, is a significant generalization of the random classification noise model, in that the adversary is allowed to flip the label of each example $x$ with a different probability $\eta(x)$, with the only constraint $\eta(x) \leq \eta$ for a certain parameter $\eta \in \big[0, \frac12\big)$. Due to its highly asymmetric nature, it remains elusive to develop computationally efficient algorithms that are robust to bounded noise. As a matter of fact, the well-known averaging scheme~\cite{kearns1998efficient} and one-shot convex loss minimization are both unable to guarantee excess error arbitrarily close to zero even with infinite supply of training examples~\cite{awasthi2015efficient,awasthi2016learning,diakonikolas2019distribution}. Therefore, a large body of recent works are devoted to designing more sophisticated algorithms to tolerate bounded noise, see, for example, \cite{awasthi2015efficient,awasthi2016learning,zhang2017hitting,yan2017revisiting,zhang2018efficient,diakonikolas2019distribution,diakonikolas2020learning}.

{\bfseries Label-efficient learning.} \
Motivated by many practical applications in which there are massive amounts of unlabeled data that are expensive to annotate, active learning was proposed as a paradigm to mitigate labeling costs~\cite{cohn1994improving,dasgupta2011two}. In contrast to traditional supervised learning (also known as passive learning) where the learner is presented with a set of labeled training examples, in active learning, the learner starts with a set of unlabeled examples, and is allowed to make label queries during the learning process~\cite{cohn1994improving,dasgupta2005coarse}. By adaptively querying examples whose labels are potentially most informative, a classifier of desired accuracy can be actively learned while requiring substantially less label feedback than that of passive learning under broad classes of data distributions ~\cite{hanneke2014theory,balcan2016active}.

{\bfseries Attribute-efficient learning.} \
With the unprecedented growth of high-dimensional data generated in biology, economics, climatology, and other fields of science and engineering, it has become ubiquitous to leverage extra properties of the data into algorithmic design for more sample-efficient learning~\cite{littlestone1987learning}. On the computational side, the goal of attribute-efficient learning is to find a {\em sparse} model that identifies most useful features for prediction~\cite{fan2008high}. On the statistical side, the focus is on answering when and how learning of a sparse model will lead to improved performance guarantee on sample complexity, generalization error, or mistake bound. These problems have been investigated for a long term, and the sparsity assumption proves to be useful for achieving non-trivial guarantees~\cite{blum1990learning,tibshirani1996regression,chen1998atomic}. The idea of attribute-efficient learning was also explored in a variety of other settings, including online classification~\cite{littlestone1987learning}, learning decision lists~\cite{servedio1999computational,servedio2012attribute,klivans2004toward,long2006attribute}, and learning parities and DNFs~\cite{feldman2007attribute}.



\vspace{0.1in}
In this work, we consider computationally efficient learning of halfspaces in all three aspects above. Specifically, we study active learning of sparse halfspaces under the more-realistic bounded noise model, for which there are a few recent works that are immediately related to ours but under different degrees of noise tolerance and distributional assumptions. In the membership query model~\cite{angluin1988queries}, where the learner is allowed to synthesize intances for label queries, ~\cite{chen2017near} proposed an algorithm that tolerates bounded noise with near-optimal $\tilde{\order}\big(\frac{d}{(1-2\eta)^2} \ln\frac1\epsilon\big)$ label complexity.
In the more realistic PAC active learning model~\cite{kearns1994toward,balcan2009agnostic}, where the learner is only allowed to query the label of the examples drawn from the unlabeled data distribution, less progress is made towards optimal performance guarantee.
Under the assumption that the unlabeled data distribution is uniform over the unit sphere, \cite{yan2017revisiting} proposed a Perceptron-based active learning algorithm that tolerates any noise rate of $\eta \in \big[0, \frac12\big)$, with label complexity of $\tilde{\order}\big(\frac{d}{(1-2\eta)^2} \ln\frac1\epsilon \big)$. Unfortunately, it is challenging to generalize their analysis beyond the uniform distribution, as their argument heavily relies on its symmetry.
Under the broader isotropic log-concave distribution over the unlabeled data, the state-of-the-art results provide much worse label complexity bounds for the bounded noise model. Specifically, \cite{awasthi2015efficient} showed that\footnote{\cite{awasthi2015efficient} phrased all the results with respect to the uniform distribution of the unlabeled data. However, their analysis can be straightforwardly extended to isotropic log-concave distributions, and was spelled out in \cite{zhang2018efficient}.} by sequentially minimizing a series of localized hinge losses, an algorithm can tolerate bounded noise up to a constant noise rate $\eta \approx 2 \times 10^{-6}$. Furthermore, \cite{awasthi2016learning} combined polynomial regression~\cite{kalai2008agnostically} and margin-based sampling~\cite{balcan2007margin} to develop algorithms that tolerate $\eta$-bounded noise for any $\eta \in \big[0, \frac12\big)$. However, their label complexity scales as $\awasthifull$, which is exponential in $\frac{1}{1-2\eta}$ and is polynomial in $d$ only when $\eta$ is away from $\frac12$ by a constant. This naturally raises our {\em first question}:~can we design a computationally efficient algorithm for active learning of halfspaces, such that for any $\eta \in \big[0, \frac12\big)$, it has a $\poly{d, \ln\frac1\epsilon, \frac{1}{1-2\eta}}$ label complexity under the more general isotropic log-concave distributions?


Compared to the rich literature of active learning of general non-sparse halfspaces, there are relatively few works on active learning of halfspaces that both exploit the sparsity of the target halfspace and are tolerant to bounded noise. Under the assumption that the Bayes classifier is an $s$-sparse halfspace (where $s \ll d$), a few active learning algorithms have been developed.
In the membership query model, a composition of the support recovery algorithm developed in~\cite{haupt2011robust} with the full-dimensional active learning algorithm~\cite{chen2017near} yields a procedure that can tolerate $\eta$-bounded noise with information-theoretic optimal $\tilde{\order}\del{\frac{s}{(1-2\eta)^2} (\ln d + \ln\frac1\epsilon) }$ label complexity.
Under the PAC active learning model where the unlabeled data distribution is isotropic log-concave, \cite{awasthi2016learning} presented an efficient algorithm that has a label complexity of $\awasthisparse$. Under the additional assumption that $\eta$ is smaller than a numerical constant substantially bounded away from $\frac12$, \cite{zhang2018efficient} gave an algorithm that has label complexity of $\tilde{\order}\del{ s \cdot \polylog{d,\frac1\epsilon}}$. Neither of these two works obtained a label complexity bound that is polynomial in $\frac{1}{1-2\eta}$ (specifically, the latter work has no guarantees when $\eta$ is greater than a constant, say $1/4$). This raises our {\em second question}:~if the Bayes classifier is an $s$-sparse halfspace, can we design an efficient halfspace learning algorithm which not only works for any bounded noise rate $\eta \in \big[0, \frac12\big)$, but also enjoys a label complexity of $\poly{s, \ln d, \ln\frac1\epsilon, \frac{1}{1-2\eta}}$?




\subsection{Summary of our contributions}
In this work, we answer both of the above questions in the affirmative. Specifically, we focus on the setting where the unlabeled data are drawn from an isotropic log-concave distribution, and the label noise satisfies the $\eta$-bounded noise condition for any $\eta \in [0, \frac{1}{2})$. We develop an attribute-efficient learning algorithm that runs in polynomial time, and achieves a label complexity of $\tilde{\order}\del{ \frac{s}{(1-2\eta)^4} \cdot  \polylog{d,\frac 1 \epsilon} }$ 
provided that the underlying Bayes classifier is an $s$-sparse halfspace. Our results therefore substantially improve upon the state-of-the-art label complexity of $\awasthisparse$ 
in the same setting~\cite{awasthi2016learning}. Even in the non-sparse setting (by letting $s=d$), our label complexity bound $\tilde{\order}\del{ \frac{d}{(1-2\eta)^4} \cdot  \ln{\frac 1 \epsilon} }$ is the first one of order $\poly{d, \ln\frac1\epsilon, \frac{1}{(1-2\eta)}}$. Prior to this work,  the best label complexity is $\awasthifull$~\cite{awasthi2016learning}. 
We summarize and compare our results in active learning to the state of the art in Tables~\ref{tb:comp-sparse} and~\ref{tb:comp-nonsparse}, in the sparse and non-sparse setting, respectively.

As a side result of our main discoveries, our algorithm also achieves a state-of-the-art sample complexity of $\tilde{\order}\del{ \frac{d}{(1-2\eta)^3} \del{\frac{1}{(1-2\eta)^3} + \frac1\epsilon} }$ for passive learning of $d$-dimensional halfspaces under the same assumptions of noise and data distribution.
In an independent and concurrent work~\cite{diakonikolas2020learning}, an efficient (passive) halfspace learning algorithm that tolerates $\eta$-bounded noise has been developed, under a broader family of unlabeled data distributions. Specializing their result to the setting when the unlabeled distribution is isotropic log-concave, their algorithm has a higher sample complexity of $\order\del{\frac{d^9}{\epsilon^4 (1-2\eta)^{10}}}$. Our techniques are also very different from theirs: they propose to find an approximate first-order stationary point of an empirical average of a single nonconvex loss function, which guarantees closeness to the Bayes classifier in angle; in contrast, we use online learning techniques to indirectly optimize a sequence of proximity functions, and additionally utilize the power of active learning in the learning process -- see Section~\ref{sec:techniques} for more details.
We discuss the implications of our work for passive learning in Section~\ref{sec:passive}, and additional related works in Appendix~\ref{sec:addl-rw}.

\begin{table}[t]
\caption{A comparison of our result to prior state-of-the-art works on efficient active learning of sparse halfspaces with $\eta$-bounded noise, where the unlabeled data distribution is isotropic log-concave.}\label{tb:comp-sparse}
\vspace{0.05in}
\centering
\begin{tabular}{lccccc}
\toprule
Work & Tolerates any $\eta \in \big[0, \frac12\big)$? & Label complexity \\
\midrule
\cite{zhang2018efficient} &  $\xmark$ & $\tilde{\order}\del{s \cdot \polylog{d, \frac1\epsilon}}$  for small enough $\eta$\\
\cite{awasthi2016learning} &  $\cmark$ & $\awasthisparse$ \\
{\bf This work} & $\cmark$ & $\tilde{\order}\del{ \frac{s}{(1-2\eta)^4} \polylog{d, \frac 1 \epsilon} }$ \\
\bottomrule
\end{tabular}
\end{table}

\begin{table}[t]
\caption{A comparison of our result to prior state-of-the-art works on active learning of non-sparse halfspaces with $\eta$-bounded noise, where the unlabeled data distribution is isotropic log-concave.  
}\label{tb:comp-nonsparse}
\vspace{0.05in}
\centering
\begin{tabular}{lccccc}
\toprule
Work & Tolerates any $\eta \in \big[0, \frac12\big)$? & Label complexity\\
\midrule
\cite{awasthi2015efficient}  & $\xmark$ & $\tilde{\order}\del{d \ln\frac{1}{\epsilon}}$  for small enough $\eta$ \\
\cite{awasthi2016learning}  & $\cmark$ & $\awasthifull$\\
{\bf This work}  & $\cmark$ & $\tilde{\order}\del{ \frac{d}{(1-2\eta)^4} \cdot \ln{\frac 1 \epsilon} }$ \\
\bottomrule
\end{tabular}
\end{table}


\subsection{An overview of our techniques}
\label{sec:techniques}
We discuss the main techniques we developed in this paper below.

{\bfseries 1) Active learning via regret minimization.} 
We approach the active halfspace learning problem with a novel interpretation of online learning regret inequalities. Consider $v \in \RR^d$, a vector that has angle at most $\theta$ with the underlying Bayes optimal halfspace $u$; our goal is to refine $v$ to $v'$, such that $v'$ has angle at most $\theta/2$ with $u$. Our key idea is to design an appropriate online linear optimization problem~\cite[Section 2.3]{orabona2019modern} with a sequence of adaptively-chosen loss functions $\cbr{w \mapsto \inner{g_t}{w}}_{t = 1}^T$ so that we can extract a good $v'$ from a sequence of regret-minimizing $w_t$'s: if we can show 
\begin{equation}
\sum_{t=1}^T \inner{w_t}{g_t} - \sum_{t=1}^T \inner{u}{g_t}
\leq 
\Reg(T),
\label{eqn:reg-ineq}
\end{equation}
for some $\Reg(T) = O(\sqrt{T})$ (say), then for large enough $T$, (roughly) an average of $w_t$ can serve as a $v'$ that satisfies our target angle proximity with $u$. 
To this end, we carefully and adaptively construct gradients $g_t$, such that: (a)~$\abs{\inner{w_t}{g_t}}$ is small; and (b) each {\em negative benchmark} term $\inner{u}{-g_t}$ has conditional expectation $\EE \sbr{\inner{u}{-g_t} \mid w_t}$ that upper bounds $f_{u,b}(w_t)$, for some function $f_{u,b}(w)$ that measures the distance between the input vector $w$ and $u$. These properties of $g_t$, along with the regret inequality~\eqref{eqn:reg-ineq} ensure that the average of $f_{u,b}(w_t)$'s is small for large $T$, which, via a nonstandard online to batch conversion, implies that an average of $w_t$ is close to $u$. To additionally achieve attribute efficiency, we use online mirror descent with well-known sparsity-inducing regularizers~\cite[Section 6]{orabona2019modern}, to guarantee that $\Reg(T)$ is smaller for sparse $u$. See Algorithm~\ref{alg:refine} and Theorem~\ref{thm:refine} for more details.


{\bfseries 2) A new update rule that tolerates bounded noise.} As discussed above, a key step  in the above regret minimization argument is to define the gradient $g_t$ such that $\EE \sbr{\inner{u}{-g_t} \mid w_t} \geq f_{u,b}(w_t)$.
For each iterate $w_t$, we choose to sample labeled example $(x_t, y_t)$ from the data distribution $D$ conditioned on the band $\cbr{x: \abs{\inner{\frac{w_t}{\| w_t \|}}{x}} \leq b}$, similar to~\cite{dasgupta2009percerptron}.
Based on labeled example $(x_t, y_t)$, a natural choice is $g_t = -\ind{\hat{y}_t \neq y_t} y_t x_t$, i.e. the negative Perceptron update, where $\hat{y}_t = \sign{\inner{w_t}{x_t}}$. Unfortunately, due to the asymmetry of the unlabeled data distribution\footnote{~\cite{yan2017revisiting} extensively utilizes the symmetry of the uniform unlabeled distribution to guarantee that the expectation is positive if the angle between $w_t$ and $u$ is large; we cannot use this as we are dealing with a more general family of log-concave unlabeled distribution, which can be highly asymmetric.}, it does not have the property we desire (in fact, the induced $f_{u,b}(w_t)$ can be negative with such choice of $g_t$). To cope with this challenge, we propose a novel setting of $g_t$ that takes into account the bounded noise rate $\eta$:
\begin{equation*}
g_t = -\ind{\hat{y}_t \neq y_t} y_t x_t - \eta \hat{y}_t x_t
 = \del{-\frac12 y_t + \del{\frac12 - \eta} \hat{y}_t} x_t.
\end{equation*}
Observe that the above choice of $g_t$ is more aggressive than the Perceptron update, in that when $\eta > 0$, even if the current prediction $\hat{y}_t$ matches the label returned by the oracle, we still update the model. In the extreme case that $\eta = 0$, we recover the Perceptron update. We show that, this new setting of $g_t$, in conjunction with the aforementioned adaptive sampling scheme, yields a function $f_{u,b}(w)$ that possesses desirable properties. We refer the reader to Lemma~\ref{lem:update-distance} for a precise statement.

{\bfseries3) Averaging-based initialization that exploits sparsity.}
The above arguments suffice to establish a local convergence guarantee, i.e. given a vector $\tilde{v}_0$ with $\theta(\tilde{v}_0, u) \leq \frac\pi{32}$, one can repeatedly run a sequence of online mirror descent updates and online-to-batch conversions, such that for each $k \geq 0$, we obtain a vector $\tilde{v}_k$ such that
$\theta(\tilde{v}_k, u) \leq \frac{\pi}{32 \cdot 2^k}$.
It remains to answer the question of how to obtain such $\tilde{v}_0$ using active learning in an attribute-efficient manner. To this end, we design an initialization procedure that finds such $\tilde{v}_0$ with $\tilde{\order}\del{\frac{s}{(1-2\eta)^4} \cdot \polylog{d}}$ labeled examples. It consists of two stages.
The first stage performs the well known averaging scheme~\cite{kearns1998efficient}, in combination with a novel hard-thresholding step~\cite{blumensath2009iterative}. This stage gives a unit vector $w^\sharp$ such that $\inner{w^\sharp}{u} \geq \Omega(1-2\eta)$ with high probability, using $\tilde{\order}\del{\frac{s \ln d}{(1-2\eta)^2}}$ labeled examples.
The second stage performs online mirror descent update with adaptive sampling as before, but with the important constraint that $\inner{w_t}{w^\sharp} \geq \Omega(1-2\eta)$ for all iterates $w_t$. Through a more careful analysis using the function $f_{u,b}$ discussed above (that accounts for the case where input $w_t$ can have a large obtuse angle with $u$), we obtain a vector $\tilde{v}_0$ that has the desired angle upper bound, with the aforementioned label complexity. We refer the reader to Lemma~\ref{lem:w-sharp} and Theorem~\ref{thm:init} for more details.

%% file: prelims.tex
\section{Preliminaries}
{\bfseries Active learning in the PAC model.} 
We consider active halfspace learning in the agnostic PAC learning model~\cite{kearns1994toward,balcan2009agnostic}. In this setting, there is an instance space $\calX = \RR^d$ where all examples' features take value from, and a label space $\calY = \{ -1, 1 \}$ where all examples' labels take value from. The data distribution $D$ is a joint probability distribution over $\calX \times \calY$. We denote by $D_X$ the marginal distribution of $D$ on $\Xcal$, and by $D_{Y|X=x}$ the conditional distribution of $Y$ given $X=x$. We will also refer to $D_X$ as unlabeled data distribution. Throughout the learning process, the active learner is given access to two oracles: $\EX$, an unlabeled example oracle that returns $x$ randomly drawn from $D_X$, and $\Ocal$, a labeling oracle takes $x$ as input and returns a label $y \sim D_{Y|X=x}$.

A classifier is a mapping from $\Xcal$ to $\Ycal$. We consider halfspace classifiers of the form $h_w: x \mapsto \sign{w \cdot x}$ where $\sign{z} = +1$ if $z \geq 0$ and equals $-1$ otherwise. The vector $w \in \RR^d$ is the parameter of $h_w$, which has unit $\ell_2$-norm. For a given classifier $h_w$, we measure its performance by $\err(h_w, D) := \PP_{(x,y) \sim D}\del{h_w(x) \neq y}$, i.e. the probability that a random example gets misclassified.

We are interested in developing active halfspace learning algorithms that achieve the agnostic PAC guarantee. Specifically, we would like to design an algorithm $\Acal$, such that it receives as inputs excess error parameter $\epsilon \in (0, 1)$ and failure probability parameter $\delta \in (0, 1)$, and
with probability $1-\delta$, after making a number of queries to $\EX$ and $\Ocal$, $\Acal$ returns a halfspace $h_w$ such that $\err(h_w, D) - \min_{w'} \err(h_{w'}, D) \leq \epsilon$. In addition, we would like our active learner to make as few label queries as possible. We denote by $n_{\Acal}(\epsilon, \delta)$ the number of label queries of $\Acal$ given parameters $\epsilon$ and $\delta$; this is also called the {\em label complexity} of $\Acal$.

We will focus on sampling unlabeled examples from $D_X$ conditioned on a subset $B$ of $\RR^d$; this can be done by rejection sampling, where we repeatedly call $\EX$ until we see an unlabeled example $x$ falling in $B$.
Given a unit vector $\hat{w}$ and $b > 0$, define $B_{\hat{w}, b} = \cbr{x \in \RR^d: \abr{ \hat{w} \cdot x } \leq b}$. Denote by $D_{X \mid \hat{w}, b}$ (resp. $D_{\hat{w}, b}$) the distribution of $D_X$ (resp. $D$) conditioned on the event that $x \in B_{\hat{w}, b}$.

%

{\bfseries Vectors.} 
Let $w$ be a vector in $\RR^d$. The $\ell_0$-``norm'' $\| w \|_0$ counts its number of nonzero elements, and $w \in \RR^d$ is said to be $s$-sparse if $\| w \|_0 \leq s$. Given $s \in \cbr{1,\ldots, d}$, the hard thresholding operation $\Hcal_s(w)$ zeroes out all but $s$ largest (in magnitude) entries of $w$. 
For a scalar $\gamma \geq 1$, denote by $\| w \|_\gamma$ the $\ell_{\gamma}$-norm of the vector $w$. If not explicitly mentioned, $\| \cdot \|$ denotes the $\ell_2$-norm.  We denote by $\hat{w} = \frac{w}{\| w\|}$ as the $\ell_2$-normalization of $w \in \Rd$. For two vectors $w_1, w_2$, we write $\theta(w_1,w_2) = \arccos\del{\hat{w}_1 \cdot \hat{w}_2}$ as the angle between them.

{\bfseries Convexity.}
Given a convex and differentiable function $f$, its induced Bregman divergence is given by $D_f(w, w') \defeq f(w) - f(w') - \inner{\nabla f(w')}{w - w'}$. Note that by the convexity of $f$, $D_f(w, w') \geq 0$ for all $w$ and $w'$. A function $f$ is said to be $\lambda$-strongly convex with respect to the norm $\| \cdot \|_{\gamma}$, if $D_f(w, w') \geq \frac{\lambda}{2}\| w - w'\|^2_{\gamma}$ holds for all $w$ and $w'$ in its domain. In our algorithm, we will use the following convex function: 
$
\Phi_{v}(w) \defeq \frac{1}{2(p-1)}\| w - v \|_p^2,
$
where $v$ is a reference vector in $\R^d$. Throughout the paper, we reserve $p$ for a specific value $p = \frac{\ln (8d)}{\ln(8d) - 1}$, and reserve $q = \ln (8d)$ (note that $p^{-1} + q^{-1} = 1$). As $p \in (1,2]$, $\Phi_v$ is $1$-strongly convex with respect to $\| \cdot \|_p$~\cite[Lemma 17]{shalev2007online}. In addition, $\nabla\Phi_v$ is a one-to-one mapping from $\RR^d$ to $\RR^d$, and hence has an inverse, denoted as $\nabla\Phi_v^{-1}$.


{\bfseries Distributional assumptions.} 
Without distributional assumptions, it is known that agnostically learning halfspaces is computationally hard~\cite{feldman2006new,guruswami2009hardness}. 
We make the following two assumptions.

\begin{assumption}\label{a:bn}
The data distribution $D$ satisfies the {\em $\eta$-bounded noise condition} with respect to an $s$-sparse unit vector $u \in \RR^d$, where the noise rate $\eta \in [0, 1/2)$. Namely, for all $x \in \Xcal$, 
$\PP(y \neq \sign{{u} \cdot {x}} | X = x) \leq \eta$.
\end{assumption}

\begin{assumption}\label{a:ilc}
The unlabeled data distribution $D_X$ is {\em isotropic log-concave over $\RR^d$}, i.e. $D_X$ has a probability density function $f$ over $\RR^d$ such that $\ln f(x)$ is concave, and $\EE_{x \sim D_X} \sbr{x x^\top} = I_{d \times d}$.
\end{assumption}

Assumption~\ref{a:bn} implies that the Bayes optimal classifier with respect to the distribution $D$ is $h_u$. As a consequence, the optimal halfspace is $h_u$, namely $\err(h_u, D) = \min_{w'} \err(h_w, D)$. Assumption~\ref{a:ilc} has appeared in many prior works~\cite{klivans2009learning,balcan2013active,awasthi2017power,zhang2018efficient}. \cite{balcan2013active} showed the following important lemma.
\begin{lemma}\label{lem:angle-disag}
Suppose that Assumption~\ref{a:ilc} is satisfied. There exist absolute constants $c_1$ and $c_2$, such that for any two vectors $v_1$ and $v_2$,
\begin{equation*}
c_1 \PP_{x \sim D_X}(\sign{v_1 \cdot x} \neq \sign{v_2 \cdot x}) \leq \theta(v_1, v_2) \leq c_2  \PP_{x \sim D_X}(\sign{v_1 \cdot x} \neq \sign{v_2 \cdot x}).
\end{equation*}
\end{lemma}

%% file: algorithm.tex
\section{Main Algorithm}

We present Algorithm~\ref{alg:main}, our noise-tolerant attribute-efficient active learning algorithm, in this section. It consists of two stages: an initialization stage \init (line~\ref{line:init}) and an iterative refinement stage (lines~\ref{line:refine-start} to~\ref{line:refine-end}). In the initialization stage, we aim to find a vector $\tilde{v}_0$ such that $\theta(\tilde{v}_0, u) \leq \frac\pi{32}$; 
in the iterative refinement stage, we aim to bring our iterate $\tilde{v}_k$ closer to $u$ after each phase $k$. Specifically, suppose that $\theta(\tilde{v}_{k-1}, u) \leq \frac{\pi}{32 \cdot 2^{k-1}}$ at the beginning of iteration $k$, then after the execution of line~\ref{line:refine}, we aim to obtain a new iterate $\tilde{v}_k$ such that $\theta(\tilde{v}_k, u) \leq \frac{\pi}{32 \cdot 2^k}$ with high probability. 
The iterative refinement stage ends when $k$ reaches $k_0$, in which case we are guaranteed that $\optu = \tilde{v}_{k_0}$ is such that
$\theta(\optu, u) \leq \frac{\pi}{32 \cdot 2^{k_0}} \leq c_1 \epsilon$, where $c_1$ is the constant defined in Lemma~\ref{lem:angle-disag}. From Lemma~\ref{lem:angle-disag}, we have that $\PP_{x \sim D_X}(h_{\optu}(x) \neq h_u(x)) \leq \epsilon$. Consequently, by triangle inequality, we have that
$\err(h_{\optu}, D) - \err(h_u, D) \leq \PP_{x \sim D_X}(h_{\optu}(x) \neq h_u(x)) \leq \epsilon$.


\begin{algorithm}[h]
\caption{Main algorithm}
\label{alg:main}
\begin{algorithmic}[1]
\REQUIRE Target error $\epsilon$, failure probability $\delta$, bounded noise level $\eta$, sparsity $s$.
\ENSURE Halfspace $\optu$ in $\RR^d$ such that $\err(h_{\optu}, D) - \err(h_u, D) \leq \epsilon$.
\STATE Let $k_0 = \lceil \log \frac1{c_1 \epsilon} \rceil$ be the total number of iterations, where $c_1$ is defined in Lemma~\ref{lem:angle-disag} .
\STATE Let $\tilde{v}_0 \gets \init(\frac{\delta}{2}, \eta, s)$.  // {See Algorithm~\ref{alg:init}.}
\label{line:init}
\FOR{phases $k=1, 2, \dots, k_0$}
\label{line:refine-start}
\STATE $v_{k-1} \gets \calH_s(\tilde{v}_{k-1})$.
\STATE $\tilde{v}_k \leftarrow \refine(v_{k-1}, \frac{\delta}{2 k(k+1)}, \eta, s, \alpha_k, b_k, \Kcal_k, R_k, T_k)$, where the
step size $\alpha_k = \tilde{\Theta}\del{(1-2\eta) 2^{-k}}$, bandwidth $b_k = \Theta\del{(1-2\eta) 2^{-k}}$,
constraint set 
\begin{equation*}
\Kcal_k = \cbr{w \in \RR^d: \| w - v_{k-1} \|_2 \leq \pi\cdot 2^{-k-3}, \| w \|_2 \leq 1 },
\end{equation*}
regularizer $R_k(w) = \Phi_{v_{k-1}}(w)$, number of iterations $T_k = \order\del{\frac{s}{(1-2\eta)^2} \big(\ln \frac{d \cdot k^2 2^k}{\delta (1-2\eta)}\big)^3 }$. \label{line:refine}
\ENDFOR
\label{line:refine-end}
\RETURN $\optu \gets \tilde{v}_{k_0}$. 
\end{algorithmic}
\end{algorithm}


\begin{algorithm}[h]
\caption{\refine}
\label{alg:refine}
\begin{algorithmic}[1]
\REQUIRE Initial halfspace $w_1$, failure probability $\delta'$, bounded noise level $\eta$, sparsity $s$, learning rate $\alpha$, bandwidth $b$, convex constraint set $\Kcal$, regularization function $R(w)$, number of iterations $T$.
\ENSURE Refined halfspace $\tilde{w}$.
\FOR{$t=1, 2, \dots, T$}
\label{line:omd-loop-start}
\STATE Sample $x_t$ from $D_{X \mid \hat{w}_t, b}$, the conditional distribution of $D_X$ on $B_{\hat{w}_t, b}$ and query $\calO$ for its label $y_t$ (recall that $\hat{w}_t$ is the $\ell_2$-normalization of $w_t$).
\STATE Update:
$w_{t+1} \gets\arg\min_{w \in \Kcal} D_{R}\(w, \nabla R^{-1} \(\nabla R(w_t) - \alpha g_t \) \)$,
where the gradient $g_t = \del{ -\frac12 y_t + \del{\frac12 - \eta} \hat{y}_t } x_t$,  and $\hat{y}_t = \sign{w_t \cdot x_t }$.
\label{line:omd}
\ENDFOR
\label{line:omd-loop-end}
\STATE $\bar{w} \gets \frac{1}{T} \sum_{t=1}^T \hat{w}_t$.
\RETURN $\tilde{w} \leftarrow \frac{\bar{w}}{\| \bar{w} \|}$.
\end{algorithmic}
\end{algorithm}
\begin{algorithm}
\caption{\init}
\label{alg:init}
\begin{algorithmic}[1]
\REQUIRE Failure probability $\delta'$, bounded noise parameter $\eta$, sparsity parameter $s$.
\ENSURE Halfspace $\tilde{v}_0$ such that $\theta(\tilde{v}_0, u) \leq \frac{\pi}{32}$.
\STATE $(x_1,y_1),\ldots,(x_m, y_m) \gets$ draw $m$ examples iid from $D_X$, and query $\calO$ for their labels, where $m = 81 \cdot 2^{51} \cdot \frac{s \ln \frac{8d}{\delta'}}{(1-2\eta)^2}$.
\STATE Compute $w_{\text{avg}} = \frac 1 m \sum_{i=1}^m x_i y_i$.
\label{line:init-avg}
\STATE Let $w^\sharp = \frac{\calH_{\tilde{s}}(w_{\text{avg}})}{\| \calH_{\tilde{s}}(w_{\text{avg}}) \|}$, where $\tilde{s} = \frac{81\cdot 2^{38}}{(1-2\eta)^2}s$.
\label{line:init-ht}
\STATE Find a point $w_1$ in the set
$\Kcal = \cbr{w: \| w \|_2 \leq 1, \| w \|_1 \leq \sqrt{s}, \inner{w}{w^\sharp} \geq \frac{(1-2\eta)}{9 \cdot 2^{19}}}$.
\RETURN $\tilde{v}_0 \gets \refine(w_1, \frac{\delta'}{2}, \eta, s, \alpha, b, \Kcal, R, T)$, where step size $\alpha = \tilde{\Theta}\del{(1-2\eta)^2}$, bandwidth $b = \Theta\del{(1-2\eta)^2}$,
constraint set $\Kcal$,
regularizer $R(w) = \Phi_{w_1}(w)$, and number of iterations $T = \order\del{\frac{s}{(1-2\eta)^4}\del{\ln\frac{d}{\delta' (1-2\eta)}}^3}$.
\end{algorithmic}
\end{algorithm}




{\bfseries The refinement procedure.} We first describe our refinement procedure, namely Algorithm~\ref{alg:refine}, in detail. When used by Algorithm~\ref{alg:main}, it requires that the input $w_1$ has angle $\theta \in [0,\frac{\pi}{32}]$ with $u$, and aims to find a new $\tilde{w}$ such that it has angle $\theta/2$ with $u$. It performs iterative update on $w_t$'s (lines~\ref{line:omd-loop-start} to~\ref{line:omd-loop-end}) in the following manner. Given the current iterate $w_t$, it defines a (time-varying) sampling region $B_{\hat{w}_t, b}$, samples an example $x_t$ from $D_X$ conditioned on $B_{\hat{w}_t, b}$, and queries its label $y_t$. This time-varying sampling strategy has appeared in many prior works on active learning of halfspaces, such as~\cite{dasgupta2009percerptron,yan2017revisiting}.

Then, given the example $(x_t, y_t)$, it performs an online mirror descent update (line~\ref{line:omd}) with regularizer $R(w)$, along with a carefully designed update vector $-\alpha g_t$. The gradient vector
\begin{equation*}
g_t = \del{ -\frac12 y_t + \del{\frac12 - \eta} \hat{y}_t } x_t
= \begin{cases} -\eta y_t x_t, & y_t = \hat{y}_t, \\ -(1 - \eta) y_t x_t, & y_t \neq \hat{y}_t, \end{cases}
\end{equation*}
is a carefully-scaled version of $- y_t x_t$. Observe that if $\eta = 0$, i.e. the noise-free setting, our algorithm sets $g_t = -\ind{\hat{y}_t \neq y_t} y_t x_t$, which is the gradient widely used in online classification algorithms, such as Perceptron~\cite{rosenblatt1958perceptron}, Winnow~\cite{littlestone1987learning} and $p$-norm algorithms~\cite{grove2001general,gentile2003robustness}. As we shall see, this modified update is important to the algorithm's bounded noise tolerance (Lemma~\ref{lem:update-distance}). Observe that Algorithm~\ref{alg:refine} is computationally efficient, as each step of online mirror descent update only requires solving a convex optimization problem; specifically, $\Kcal$ is a convex set, and $D_R(\cdot, \cdot)$ is convex in its first argument.

In the calls of Algorithm~\ref{alg:refine} in Algorithm~\ref{alg:main}, the constraint set $\Kcal_k$ is different from the one in~\cite{zhang2018efficient}, where an additional $\ell_1$ constraint is used and is crucial for near-optimal dependence on the sparsity and dimension. Here $\ell_1$ constraints are not necessary. In fact, when invoking Algorithm~\ref{alg:refine}, we use regularizer $R(w)$ of form $\Phi_v(w) = \frac{1}{2(p-1)}\| w - v \|_p^2$ for $p = \frac{\ln (8d)}{\ln (8d) - 1}$, which is well known to induce attribute efficiency~\cite{grove2001general,gentile2003robustness}. See Appendix~\ref{sec:local-convergence} for a formal treatment.



After obtaining the iterates $\{w_t\}_{t=1}^T$, we tailor online-to-batch conversion~\cite{cesa2004generalization} to our problem: we take an average over the $\ell_2$-normalized $w_t$'s, and further normalize it to obtain our refined estimate $\tilde{w}$.



{\bfseries The initialization procedure.} Our initialization procedure, Algorithm~\ref{alg:init}, aims to produce a vector $\tilde{v}_0$ such that $\theta(\tilde{v}_0, u) \leq \frac\pi{32}$. It consists of two stages. At its first stage, it generates a very coarse estimate of $u$, namely $w^\sharp$, as follows: first, we take the average of $x_i y_i$'s to obtain $w_{\text{avg}}$ (line~\ref{line:init-avg}); next, it performs  hard-thresholding and normalization on $w_{\text{avg}}$ (line~\ref{line:init-ht}), with parameter $\tilde{s} = \order\big(\frac{s}{(1-2\eta)^2}\big)$.
As we will see, with $m = \order\big(\tilde{s} \ln d\big)$ label queries, $w^\sharp$, the output unit vector of the first stage, is such that $\inner{w^\sharp}{u} \geq \Omega(1-2\eta)$. At its second stage, it uses \refine (Algorithm~\ref{alg:refine}) to obtain a better estimate, with a constraint set $\Kcal$ that incorporates the knowledge obtained at the first stage: for all $w$ in $\Kcal$, $w$ satisfies $\inner{w}{w^\sharp} \geq \Omega(1-2\eta)$. Note that $u\in \Kcal$.
Technically speaking, this additional linear constraint ensures that for all $w$ in $\Kcal$, $\theta(w, u) \leq \pi - \Omega(1-2\eta)$, which gets around technical challenges when dealing with iterates $w_t$ that are nearly opposite to $u$. See Lemma~\ref{lem:halfspace-init} in Appendix~\ref{sec:init} for more details.

We remark that it may be possible to prove a refined bound on $\theta(w_{\text{avg}},u)$ smaller than, say, $\frac{\pi}{4}$, as existing lower bounds on $\theta(w_{\text{avg}},u)$, e.g. Theorem 2 of \cite{awasthi2015efficient}, do not rule out such possibility. This could lead to a more sample-efficient initialization procedure that avoids using the above \refine procedure with the specialized setting of constraint set $\Kcal$. If this were the case, combining this with the guarantees of \refine (Theorem~\ref{thm:refine} below) would imply an active learning algorithm with information-theoretically near-optimal label complexity of $\tilde{O}(\frac{s}{(1-2\eta)^2} \polylog{d,\frac1\epsilon})$ in this setting.
We leave this as an interesting open problem.

%% file: guarantees.tex
\section{Performance Guarantees}

We now provide formal performance guarantees of Algorithm~\ref{alg:main}, showing that: 1) it is able to achieve any target excess error rate $\epsilon \in (0, 1)$; 2) it tolerates any bounded noise rate $\eta \in [0, 1/2)$; and 3) its label complexity has near-optimal dependence on the sparsity and data dimension, and has substantially improved dependence on the noise rate.

\begin{theorem}[Main result]
Suppose Algorithm~\ref{alg:main} is run under a distribution $D$ such that Assumptions \ref{a:bn} and \ref{a:ilc} are satisfied. Then with probability $1-\delta$, it returns a halfspace $\optu$ such that $\err(h_{\optu}, D) - \err(h_u, D) \leq \epsilon$. Moreover, our algorithm tolerates any noise rate $\eta \in [0, 1/2)$, and asks for a total of $\tilde{\order}\big( \frac{s}{(1-2\eta)^4} \polylog{d, \frac1\epsilon, \frac1\delta}\big)$ labels.
\label{thm:main}
\end{theorem}

The proof of this theorem consists of two parts: first, we show that with high probability, our initialization procedure returns a vector $\tilde{v}_0$ that is close to $u$, in the sense that $\|\tilde{v}_0\| = 1$ and $\theta(\tilde{v}_0, u) \leq \frac \pi {32}$ (Theorem~\ref{thm:init}); Second, we show that given such $\tilde{v}_0$, with high probability, our refinement procedure (lines~\ref{line:refine-start} to~\ref{line:refine-end}) will finally return a vector $\tilde{v}_{k_0}$ that has the target error rate $\epsilon$ (Theorem~\ref{thm:refine}). 
We defer the full proof of Theorem~\ref{thm:main} to Appendix~\ref{sec:main-proof}.
In Appendix~\ref{sec:passive}, we discuss an extension of the theorem that establishes an upper bound on the number of unlabeled examples it encounters, and discuss its implication to supervised learning.

{\bfseries Initialization step.} We first characterize the guarantees of \init in the following theorem.

\begin{theorem}[Initialization]
Suppose Algorithm~\ref{alg:init} is run under a distribution $D$ such that Assumptions \ref{a:bn} and \ref{a:ilc} are satisfied, with noise rate $\eta \in [0, 1/2)$, sparsity parameter $s$, and failure probability $\delta'$. Then with probability $1-\delta'$, it returns a unit vector $\tilde{v}_0$, such that $\theta(\tilde{v}_0, u) \leq \frac \pi {32}$. In addition, the total number of label queries it makes is $\order\big(\frac{s}{(1-2\eta)^4} \big(\ln \frac{d}{\delta'(1-2\eta)}\big)^3 \big)$.
\label{thm:init}
\end{theorem}

We prove the theorem in two steps: first, we show that lines~\ref{line:init-avg} and~\ref{line:init-ht} of Algorithm~\ref{alg:init} returns a unit vector $w^\sharp$ that has a positive inner product with $u$, specifically, $\Omega(1-2\eta)$. This gives a halfspace constraint on $u$, formally $\inner{w^\sharp}{u} \geq \Omega(1-2\eta)$. 
Next, we show that 
applying Algorithm~\ref{alg:refine} with the feasible set $\Kcal$ that incorporates the halfspace constraint, and an appropriate choice of $b$, gives a unit vector $\tilde{v}_0$ such that $\theta(\tilde{v}_0, u) \leq \frac \pi {32}$. We defer the full proof of the theorem to Appendix~\ref{sec:init}.


{\bfseries Refinement step.} Theorem~\ref{thm:refine} below shows that after hard thresholding and one step of \refine (line~\ref{line:refine}), Algorithm~\ref{alg:main} halves the angle upper bound between the current predictor $\tilde{v}_k$ and $u$. Therefore, by induction, repeatedly applying Algorithm~\ref{alg:refine} ensures $\theta(\tilde{v}_{k_0}, u) = O(\epsilon)$ with high probability.

\begin{theorem}[Refinement]
Suppose we are given a unit vector $\tilde{v}$ such that $\theta(\tilde{v}, u) \leq \theta \in [0, \frac{\pi}{32}]$. Define $v \defeq \Hcal_s(\tilde{v})$. 
Suppose Algorithm~\ref{alg:refine} is run with initial halfspace $v$,  
confidence $\delta'$, bounded noise rate $\eta$, sparsity $s$, bandwidth $b = \Theta\del{(1-2\eta) \theta}$, step size $\alpha = \tilde{\Theta}\del{(1-2\eta) \theta}$, constraint set $\Kcal = \cbr{w: \| w - v \|_2 \leq 2\theta, \| w \|_2 \leq 1 }$, regularization function $R(w) = \Phi_{v}$, number of iterations $T = \order\big(\frac{s}{(1-2\eta)^2} \big(\ln \frac{d}{\delta' \theta (1-2\eta)}\big)^3 \big)$. Then with probability $1-\delta'$, it outputs 
$\tilde{v}'$ such that $\theta(\tilde{v}', u) \leq \frac\theta2$; moreover, the total number of label queries it makes is $\order\big(\frac{s}{(1-2\eta)^2} \big(\ln \frac{d}{\delta' \theta (1-2\eta)}\big)^3 \big)$.
\label{thm:refine}
\end{theorem}



The intuition behind the theorem is as follows: we define a function $f_{u,b}(w)$ that measures the closeness between unit vector $w$ and the underlying optimal classifier $u$. As Algorithm~\ref{alg:refine} performs online mirror descent on the linear losses
$\cbr{w \mapsto \inner{g_t}{w}}_{t = 1}^T$, it achieves a regret guarantee, which implies an upper bound on the average value of $\cbr{f_{u,b}(w_t)}_{t=1}^T$. As $f_{u,b}(w)$ measures the closeness between $w$ and $u$, we can conclude that there is a overwhelming portion of $\cbr{w_t}_{t=1}^T$ that has a small angle with $u$. Consequently, by averaging and normalization, it can be argued that the resulting unit vector $\tilde{v}' = \tilde{w}$ is such that $\theta(\tilde{v}', u) \leq \frac\theta 2$. We defer the full proof of the theorem to Appendix~\ref{sec:local-convergence}.

\subsection{Implication for supervised learning}
\label{sec:passive}
In this section, we briefly outline the implication of our results to supervised learning (i.e. passive learning). As our algorithms acquire examples in a streaming fashion, it can be readily seen that, a variant of Algorithm~\ref{alg:main} can be viewed as a supervised learning algorithm: each time Algorithm~\ref{alg:main} draws unlabeled example from $D_X$, the variant immediately queries $\Ocal$ for its label. Consequently, the number of examples it encounters equals the total number of labeled examples it consumes, which corresponds to its sample complexity.

We now show that Algorithm~\ref{alg:main} uses at most
$\tilde{\order}\del{ \frac{s }{(1-2\eta)^3} \del{\frac{1}{(1-2\eta)^3} + \frac1\epsilon} \cdot \polylog{d}}$
unlabeled examples; therefore, its induced supervised learning algorithm has a sample complexity of $\tilde{\order}\del{ \frac{s}{(1-2\eta)^3} \del{\frac{1}{(1-2\eta)^3} + \frac1\epsilon} \cdot \polylog{d} }$.
Without the sparsity assumption (i.e. setting $s = d$), this yields a sample complexity of $\tilde{\order}\del{ \frac{d}{(1-2\eta)^3} \del{\frac{1}{(1-2\eta)^3} + \frac1\epsilon} }$.

\begin{theorem}\label{thm:sample-size}
Suppose that Assumptions~\ref{a:bn} and~\ref{a:ilc} are satisfied. With probability $1-\delta$, Algorithm~\ref{alg:main} makes at most $\tilde{\order}\del{ \frac{s}{(1-2\eta)^3} \del{\frac{1}{(1-2\eta)^3} + \frac1\epsilon} \cdot \polylog{d}}$ queries to the unlabeled example generation oracle $\EX$.
\end{theorem}

The proof of Theorem~\ref{thm:sample-size} can be found in Appendix~\ref{sec:proof-of-passive}.



%% file: conclusion.tex
\section{Conclusion and Discussion}
\label{sec:conclusion}
In this work we substantially improve on the state-of-the-art results on efficient active learning of sparse halfspaces under bounded noise. Furthermore, our new interpretation of online learning regret inequalities could lead to new designs of other efficient learning algorithms. 
Our algorithm has a near-optimal label complexity of $\tilde{\order}\del{\frac{s}{(1-2\eta)^2} \polylog{d, \frac{1}{\epsilon}}}$ in the local convergence phase, while having a suboptimal label complexity of $\tilde{\order}\del{\frac{s}{(1-2\eta)^4} \polylog{d}}$ in the initialization phase. 
It is still an open question whether we can obtain an efficient algorithm that achieves the information-theoretically optimal label complexity of $\tilde{\order}\del{\frac{s}{(1-2\eta)^2} \polylog{d, \frac{1}{\epsilon}}}$, possibly via suitable modifications of our initialization procedure.
It would be promising to extend our results beyond isotropic log-concave distributions~\cite{balcan2017sample}, and would be interesting to investigate whether our algorithmic insights can find applications for learning halfspaces under the Tsybakov noise model~\cite{tsybakov2004optimal} and the malicious noise model~\cite{valiant1985learning,kearns1993learning}.

%% file: related_work.tex
\section{Additional Related Works}\label{sec:addl-rw}


There is a rich literature on learning halfspaces in the presence of noise. For instance, \cite{blum1996polynomial,dunagan2008simple} studied noise-tolerant learning of halfspaces under the random classification noise model, where each label is flipped independently with probability exactly $\eta$. Their algorithm proceeds as optimizing a sequence of modified Perceptron updates, and the analysis implies that the desired halfspace can be learned in polynomial time with respect to arbitrary unlabeled distribution.  \cite{kearns1988learning} considered learning halfspaces with malicious noise, where with some probability the learner is given an adversarially-generated pair of feature vector and label. Notably, their work showed that under such noise model, it is still possible to learn a good halfspace for arbitrary data distribution in polynomial time, provided that the noise rate is $\tilde{\Omega}(\frac{\epsilon}{d})$. In a series of recent work, this bound has been significantly improved by making additional assumptions on the data distribution and more sophisticated algorithmic designs~\cite{klivans2009learning,long2011learning,awasthi2017power}. The bounded noise, also known as Massart noise~\cite{massart2006risk}, was initially studied in \cite{sloan1988types,sloan1992Corrigendum,rivest1994formal}. Very recently, \cite{diakonikolas2019distribution} presented an efficient learning algorithm that has distribution-free guarantee (albeit with vanishing excess error guarantees only in the random classification noise setting), whereas most of the prior works are built upon distributional assumptions~\cite{awasthi2015efficient,awasthi2016learning,zhang2017hitting,yan2017revisiting,zhang2018efficient}. It is worth noting that other types of noise, such as malicious noise~\cite{valiant1985learning} and adversarial noise~\cite{kearns1994toward}, have also been widely studied~\cite{kalai2008agnostically,klivans2009learning,kane2013learning,daniely2015ptas,awasthi2017power,diakonikolas2019nearly,shen2020attribute}.

There is a large body of theoretical works on active learning for general hypothesis classes; see~e.g.~\cite{dasgupta2005coarse,balcan2009agnostic,hanneke2014theory} and the references therein. Despite their generality, many of the algorithms developed are not guaranteed to be computationally efficient.
For efficient noise-tolerant active halfspace learning, aside from the aforementioned works in the main text, we also remark that the work of~\cite{balcan2013statistical} provides the first computationally efficient algorithm for halfspace learning under log-concave distribution that tolerates random classification noise, with a label complexity of $\poly{d, \ln\frac1\epsilon, \frac{1}{1-2\eta}}$. Prior to our work, it is not known how to obtain an attribute-efficient active learning algorithm with label complexity $\poly{s, \ln d, \ln \frac1\epsilon, \frac{1}{1-2\eta}}$, even under this weaker random classification noise setting.


Parallel to the development of attribute-efficient learning in learning theory, there have been a large body of theoretical works developed in compressed sensing~\cite{donoho2006compressed}. In this context, the goal is twofold:~1)~design an efficient data acquisition scheme to significantly compress a high-dimensional but effectively sparse signal; and 2)~implement an estimation algorithm that is capable of reconstructing the underlying signal from the  measurements. These two phases are bind together in view of the need of low sample complexity (i.e. number of measurements), and a large volume of theoretical results have been established to meet the goal. For instance, many of the early works utilize linear measurements for the sake of its computational efficiency, and focus on the development of effective recovery procedures~\cite{chen1998atomic,tibshirani1996regression,candes2005decoding,wainwright2009sharp,tropp2007signal,blumensath2009iterative,needell2009cosamp,foucart2011hard,zhang2011sparse,shen2017iteration,shen2018tight}. In its 1-bit variant~\cite{boufounos2008bit}, the linear measurements are further quantized to a binary code, and it bears the potential of savings of physical storage as long as accurate estimation in the low-bit setting does not require significantly more measurements. In order to account for the new data acquisition scheme, a large body of new estimation paradigms are developed in recent years. For instance, \cite{jacques2013robust} showed that exact recovery can be achieved by seeking a global optimum of a sparsity-constrained nonconvex program. \cite{plan2013one,plan2013robust,zhang2014efficient} demonstrated that $\ell_1$-norm based convex programs inherently behave as well as the nonconvex counterpart in terms of estimation error. Generally speaking, the difference between 1-bit compressed sensing and learning of halfspaces lies in the fact that in compress sensing one is able to control how the data are collected. Interestingly, \cite{knudson2016bit,baraniuk2017exponential} showed that if we manually inject Gaussian noise before quantization and pass the variance parameter to the recovery algorithm, it is possible to estimate the magnitude of the signal.

The idea of active learning is also broadly explored in the compressed sensing community under the name of adaptive sensing~\cite{haupt2009adaptive,malloy2014near}. Though \cite{castro2013fundamental} showed that adaptive sensing strategy does not lead to significant improvement on sample complexity, a lot of recent works illustrated that it does when there are additional constraints on the sensing matrix~\cite{davenport2016constrained}, or when 1-bit quantization is applied during data acquisition~\cite{baraniuk2017exponential}. As a matter of fact, \cite{baraniuk2017exponential} showed that by adaptively generating the 1-bit measurements, it is possible to design an efficient recovery algorithm that has exponential decay in reconstruction error which essentially translates into $O\(s \log(d) \log(1/\epsilon)\)$ sample complexity. 

Noisy models are also studied in compressed sensing. For instance, \cite{nguyen2013robust,chen2013robust,suggala2019adaptive} considered the situation where a fraction of the data are corrupted by outliers. \cite{plan2013robust} studied robustness of convex programs when the 1-bit measurements are either corrupted by random noise or adversarial noise.

%% file: appendix.tex
\section{Proof of Theorem~\ref{thm:main}}
\label{sec:main-proof}
In this section we present a detailed proof of Theorem~\ref{thm:main}, our main result.
\begin{proof}[Proof of Theorem~\ref{thm:main}]
We define event $E_0$ as the event that the guarantees of Theorem~\ref{thm:init} holds with failure probability $\delta' = \frac{\delta}{2}$.
In addition, we define event $E_k$ as the event that the guarantees of Theorem~\ref{thm:refine} holds for input $\tilde{v} = \tilde{v}_{k-1}$, angle upper bound $\theta = \frac{\pi}{32 \cdot 2^{k-1}}$ and output $\tilde{v}' = \tilde{v}_k$ with failure probability $\delta' = \frac{\delta}{2k(k+1)}$.
It can be easily seen that $\PP(E_0) \geq 1-\frac\delta2$, and $\PP(E_k) \geq 1 - \frac{\delta}{2k(k+1)}$ for all $k \geq 1$.

Consider event $E = \bigcap_{k=0}^{k_0} E_k$.
Using union bound, we have that $\PP(E) \geq 1 - \frac\delta 2 - \sum_{k=1}^{k_0} \frac{\delta}{2k(k+1)} \geq 1-\delta$.
On event $E$, we now show inductively that $\theta(\tilde{v}_k, u) \leq \frac{\pi}{32 \cdot 2^k}$ for all $k \in \cbr{0,1,\ldots,k_0}$. 

\paragraph{Base case.} By the definition of $E_0$ and the fact that $E \subset E_0$, we have $\theta(\tilde{v}_0, u) \leq \frac{\pi}{32}$.

\paragraph{Inductive case.} Now suppose that on event $E$, we have $\theta(\tilde{v}_{k-1}, u) \leq \frac{\pi}{32 \cdot 2^{k-1}}$. Now by the definition of event $E_k$, we have that after Algorithm~\ref{alg:refine}, we obtain a unit vector $v_k$ such that $\theta(\tilde{v}_k, u) \leq \frac{\pi}{32 \cdot 2^k}$.

This completes the induction. Specifically, on event $E$, after the last phase $k_0 = \lceil \log\frac{1}{c_1\epsilon} \rceil$, we obtain a vector
$\optu = \tilde{v}_{k_0}$, such that $\theta(\optu, u) \leq \frac{\pi}{32 \cdot 2^{k_0}} \leq c_1 \epsilon$. Now applying Lemma~\ref{lem:angle-disag}, we have that $\PP(\sign{\optu \cdot x} \neq \sign{{u}\cdot{x}}) \leq \frac1{c_1} \theta(\optu, u) \leq \epsilon$. By triangle inequality, we conclude that
\begin{equation*}
\err(h_{\optu}, D) - \err(h_u, D) \leq \PP(\sign{\optu \cdot x} \neq \sign{u \cdot x}) \leq \epsilon.
\end{equation*}

We now upper bound the label complexity of Algorithm~\ref{alg:main}. The initialization phase uses $n_0 = \order\del{\frac{s}{(1-2\eta)^4}\del{\ln\frac{d}{\delta (1-2\eta)}}^3}$ labeled queries. Meanwhile, for every $k \in [k_0]$, Algorithm~\ref{alg:refine} at phase $k$ uses $n_k = \order\del{\frac{s}{(1-2\eta)^2} \del{\ln \frac{d \cdot k^2 2^k}{\delta (1-2\eta)}}^3 }$ label queries.
Therefore, the total number of label queries by Algorithm~\ref{alg:main} is:
\begin{eqnarray*}
n = n_0 + \sum_{k=1}^{k_0} n_k
&=& \order \del{\frac{s}{(1-2\eta)^2} \del{\frac{1}{(1-2\eta)^2} \ln\frac{d}{\delta (1-2\eta)}}^3 + \ln\frac{1}{\epsilon} \cdot \del{\ln \frac{d}{\delta \epsilon (1-2\eta)}}^3} \\
&=& \order\del{\frac{s}{(1-2\eta)^4} \del{\ln\frac{d}{\delta \epsilon (1-2\eta)}}^4 } = \tilde{\order}\del{ \frac{s}{(1-2\eta)^4} \polylog{d, \frac1\epsilon, \frac1\delta}}.
\end{eqnarray*}
The proof is complete.
\end{proof}

\section{Proof of Theorem~\ref{thm:sample-size}}
\label{sec:proof-of-passive}

\begin{proof}
We first observe that if $\refine$ is run for $T$ iterations with bandwidth $b$, then with high probability, it will encounter $\order\del{\frac{T}{b}}$ unlabeled examples. This is because, $\order\del{\frac1b}$ calls of $\EX$ suffices to obtain an example that lies in $B_{\hat{w}_t, b}$, since it has probability mass $\Omega(b)$ (see Lemma~\ref{lem:ilc-density-ub}).

For the initialization step (line~\ref{line:init}), Algorithm~\ref{alg:main} first draws $\order\del{ \frac{s \ln d}{(1-2\eta)^2}}$ unlabeled examples from $D_X$;
then it runs \refine with $\tilde{\order}\del{\frac{s }{(1 - 2\eta)^4}\cdot \polylog{d} }$ iterations with bandwidth $b = \Theta\del{(1-2\eta)^2}$. Therefore, this step queries $\tilde{\order}\del{\frac{s}{(1 - 2\eta)^6} \cdot \polylog{d}}$ times to $\EX$.

Now we discuss the number of unlabeled examples in phases $1$ through $k_0$. For the $k$-th phase, Algorithm~\ref{alg:main} runs \refine with $\tilde{\order}\del{\frac{s }{(1 - 2\eta)^2}\cdot \polylog{d}}$ iterations with bandwidth $b = \Theta((1-2\eta) 2^{-k})$, which encounters $\tilde{\order}\del{\frac{s \cdot 2^k}{(1 - 2\eta)^3}\cdot \polylog{d}}$ examples.
Therefore, summing over $k = 1,2,\ldots,k_0$, the total number of unlabeled examples queried to $\EX$ is $\tilde{\order}\del{\frac{s  \cdot 2^{k_0}}{(1 - 2\eta)^3}\cdot \polylog{d}} = \tilde{\order}\del{\frac{s }{(1 - 2\eta)^3 \epsilon}\cdot \polylog{d}}$.

Summing over the two parts, the total number of queries to the unlabeled example oracle $\EX$ is $\tilde{\order}\del{\frac{s }{(1 - 2\eta)^3} \cdot \del{  \frac1{(1-2\eta)^3} + \frac1\epsilon }\cdot \polylog{d}}$.
\end{proof}

\section{Analysis of Local Convergence: Proof of Theorem~\ref{thm:refine}}
\label{sec:local-convergence}

Before delving into the proof of Theorem~\ref{thm:refine}, we first introduce an useful definition. Recall that $\hat{w}$ is the $\ell_2$-normalized vector of $w$. Define function
\begin{equation}
f_{u,b}(w) \defeq \EE_{(x,y) \sim D_{\hat{w},b}} \big[\abs{u \cdot x} \cdot \ind{\sign{w \cdot x} \neq \sign{u \cdot x}}\big].
\end{equation}
Note that for any $l > 0$ and $w$ in $\RR^d$, $f_{u,b}(w) = f_{u,b}(l w)$; specifically, $f_{u,b}(w) = f_{u,b}(\hat{w})$.
We will discuss the structure of $f_{u,b}$ in detail in Appendix~\ref{sec:f}; roughly speaking, $f_{u,b}(w)$ provides a ``distance measure'' between $w$ and $u$.


The lemma below motivates the above definition of $f_{u,b}$.
\begin{lemma}
\label{lem:update-distance}
Given a vector $w_t$ and an example $(x_t, y_t)$ sampled randomly from $D_{\hat{w}_t, b}$, define $\hat{y}_t = \sign{{w_t}\cdot {x_t}}$. Define the gradient vector induced by this example as $g_t = (-\frac12 y_t + (\frac12 - \eta)\hat{y}_t) x_t$. Then,
\begin{equation}
\EE_{x_t, y_t \sim D_{\hat{w}_t, b}} \sbr{\inner{u}{-g_t}} \geq (1-2\eta) f_{u,b}(w_t).
\end{equation}
\end{lemma}

\begin{proof}
Throughout this proof, we will abbreviate $\EE_{x_t, y_t \sim D_{\hat{w}_t, b}}$ as $\EE$. By the definition of $g_t$, we have
\begin{equation*}
\EE \sbr{\inner{u}{-g_t}}
= \EE \sbr{\frac12 y_t \inner{u}{x_t} - \del{\frac12 - \eta} \hat{y}_t \inner{u}{x_t}}.
\end{equation*}


We first look at $\EE \sbr{\frac12 y_t \inner{u}{x_t}}$. Observe that
\begin{eqnarray*}
\EE \sbr{\frac12 y_t \inner{u}{x_t}}
= \EE \sbr{\frac12 \EE[y_t \mid x_t ] \inner{u}{x_t}}
\geq \EE \sbr{\frac12 \abs{\inner{u}{x_t}} (1-2\eta)}
\end{eqnarray*}
where the equality uses the tower property of conditional expectation, and the  inequality uses Lemma~\ref{lem:u-ip-x} below.

Therefore, by linearity of expectation, along with the above inequality, we have:
\begin{align*}
&\ \EE \sbr{\frac12 y_t \inner{u}{x_t} - \(\frac12 - \eta\) \hat{y}_t \inner{u}{x_t}} \\
\geq& \(\frac{1}{2} - \eta\) \EE \big[\abs{\inner{u}{x_t}}   (1 - \sign{\inner{u}{x_t}} \sign{\inner{w}{x_t}})\big] \\
=&\ (1-2\eta) \EE \big[\abs{\inner{u}{x}} \ind{\sign{\inner{w}{x}} \neq \sign{\inner{u}{x}}}\big] \; = \; (1-2\eta) f_{u,b}(w).
\end{align*}
The lemma follows.
\end{proof}

\begin{lemma}\label{lem:u-ip-x}
Fix any $x \in \Xcal$. Suppose $y$ is drawn from $D_{Y|X=x}$ that satisfies the $\eta$-bounded noise assumption with respect to $u$. Then,
\begin{equation*}
\inner{u}{x} \EE\sbr{y \mid x} \geq (1-2\eta) \abs{\inner{u}{x}}.
\end{equation*}
\end{lemma}

\begin{proof}
We do a case analysis. If $\inner{u}{x} \geq 0$, by Assumption~\ref{a:bn}, $\PP(Y = 1|X = x) \geq 1-\eta$, making $\EE\sbr{y \mid x} = \PP(Y = 1|X = x) - \PP(Y = -1|X = x) \geq (1-2\eta)$; symmetrically, if $\inner{u}{x} < 0$, $\EE\sbr{y \mid x} \leq -(1-2\eta)$. In summary, $\inner{u}{x} \EE\sbr{y \mid x} \geq (1-2\eta) \abs{\inner{u}{x}}$.
\end{proof}







We have the following general lemma that provides a characterization of the iterates $\cbr{w_t}_{t=1}^T$ produced by Algorithm~\ref{alg:refine}.
\begin{lemma}
There exists an absolute constant $c > 0$ such that the following holds.
Suppose we are given a vector $w_1$ in $\RR^d$, convex set $\Kcal$, and scalars $r_1, r_2 > 0$ such that:
\begin{enumerate}
\item $\| w_1 - u \|_1 \leq r_1$;
\item Both $w_1$ and $u$ are in $\Kcal$;
\item For all $w$ in $\Kcal$, $\| w - u \|_2 \leq r_2$; in addition, for all $w$ in $\Kcal$, $\| w \|_2 \leq 1$.
\label{item:ell-2}
\end{enumerate}
If Algorithm~\ref{alg:refine} is run with initialization $w_1$, step size $\alpha > 0$, bandwidth $b \in [0,\frac \pi {72}]$, constraint set $\Kcal$, regularizer $R(w) = \Phi_{w_1}(w)$, number of iterations $T$, then, with probability $1-\delta$,
\begin{equation*}
\frac{1}{T} \sum_{t=1}^T f_{u,b}(w_t) \leq c \cdot \del{ \frac{\alpha \del{\ln \frac{T d}{\delta b}}^2}{(1-2\eta)} + \frac{r_1^2 \ln d}{\alpha (1-2\eta) T} + \frac{b}{(1-2\eta)} + \frac{(b + r_2)}{(1-2\eta)} \del{\sqrt{\frac{\ln\frac{1}{\delta}}T} + \frac{\ln\frac1\delta}{T}} }.
\end{equation*}
\label{lem:omd}
\end{lemma}
The proof of this lemma is rather technical; we defer it to the end of this section.


We now give an application of this lemma towards our proof of Theorem~\ref{thm:refine}.
\begin{corollary}\label{cor:refine-local}
Suppose we are given an $s$-sparse unit vector $v$ such that $\| v - u \|_2 \leq 2\theta$, where $\theta \leq \frac{\pi}{32}$.
If Algorithm~\ref{alg:refine} is run with initializer $v$, bandwidth $b = \Theta\del{(1-2\eta) \theta}$, step size $\alpha = \Theta\del{ (1-2\eta) \theta / \ln^2(\frac{d}{\delta' \theta (1-2\eta)})}$, constraint set $\Kcal = \cbr{w: \| w\|_2 \leq 1, \| w - v \|_2 \leq 2\theta}$, regularizer $R(w) = \Phi_{v}(w)$, number of iterations $T = \order\del{\frac{s}{(1-2\eta)^2} (\ln \frac{d}{\delta' \theta (1-2\eta)})^3 }$, then, with probability $1-\delta'$,
\begin{equation*}
\frac{1}{T} \sum_{t=1}^T f_{u,b}(w_t) \leq \frac{\theta}{50 \cdot 3^4 \cdot 2^{33}}.
\end{equation*}
\end{corollary}
\begin{proof}
We first check that the premises of Lemma~\ref{lem:omd} are satisfied with $w_1 = v$, $r_1 = \sqrt{8s} \theta$ and $r_2 = 4\theta$. To see this, observe that:
\begin{enumerate}
\item As both $v$ and $u$ are $s$-sparse, their difference $v - u$ is $2s$-sparse. Therefore, $\| v - u \|_1 \leq \sqrt{2s} \| v - u \|_2 \leq \sqrt{8s}\theta$; 
\item Both $u$ and $w$ are unit vectors, and have $\ell_2$ distance at most $2\theta$ to $v$, therefore they are both in $\Kcal$;
\item For all $w$ in $\Kcal$, $\| w - u \|_2 \leq \| w - v \| + \| v - u \|_2 \leq 4\theta$. Moreover, every $w$ in $\Kcal$ satisfies the constraint $\| w \|_2 \leq 1$ by the definition of $\Kcal$.
\end{enumerate}
Therefore, applying Lemma~\ref{lem:omd} with our choice of $r_1$, $r_2$, $\alpha$, $b$, and $T$, we have that, the following four terms: $\nicefrac{\alpha \del{\ln \frac{T d}{\delta' b}}^2}{(1-2\eta)}$, $\nicefrac{r_1^2 \ln d}{\alpha (1-2\eta) T}$, $\nicefrac{b}{(1-2\eta)}$, $\nicefrac{(b + r_2)}{(1-2\eta)} \del{\sqrt{\nicefrac{\ln\frac{1}{\delta'}}T} + \nicefrac{\ln\frac1{\delta'}}{T}}$, are all at most $\frac{\theta}{c \cdot 50 \cdot 3^4 \cdot 2^{35}}$. Consequently,
\begin{equation*}
\frac{1}{T} \sum_{t=1}^T f_{u,b}(w_t) \leq
c \cdot 4 \cdot \frac{\theta}{c \cdot 50 \cdot 3^4 \cdot 2^{35}} \leq \frac{\theta}{50 \cdot 3^4 \cdot 2^{33}}.
\end{equation*}
The proof is complete.
\end{proof}





We also need the following useful claim.
\begin{claim}\label{claim:f-angle}
If $\theta(w, u) \leq \frac{\pi}{2}$, and $f_{u,b}(w) \leq \frac{\theta}{5 \cdot 3^4 \cdot 2^{21}}$, then $\theta(w,u) \leq \frac{\theta}{5}$.
\end{claim}
\begin{proof}
We conduct a case analysis:
\begin{enumerate}
\item If $\theta(w,u) \leq 36b$, we are done, because from our choice of $b$, $36b \leq \frac{\theta}{5}$.
\item Otherwise, $\theta(w,u) \in [36b, \frac\pi2]$. In this case, by item~\ref{item:acute} of Lemma~\ref{lem:f-refined} in Appendix~\ref{sec:f}, we have that $f_{u,b}(w) \geq \frac{\theta(w,u)}{3^4 \cdot 2^{21}}$. In conjunction with the premise that $f_{u,b}(w) \leq \frac{\theta}{5 \cdot 3^4 \cdot 2^{21}}$, we get that $\theta(w,u) \leq \frac\theta5$.
\end{enumerate}
In summary, in both cases, we have $\theta(w,u) \leq \frac\theta5$.
\end{proof}

\begin{proof}[Proof of Theorem~\ref{thm:refine}]
First, given a unit vector $\tilde{v}$ such that $\theta(\tilde{v}, u) \leq \theta$, we have that $\| \tilde{v} - u \|_2 = 2\sin\frac{\theta(\tilde{v}, u)}{2} \leq \theta$.
As $u$ is $s$-sparse, and $v = \Hcal_s(\tilde{v})$, by Lemma~\ref{lem:best-s-approx}, we have that $\| v - u \| \leq 2\theta$.

Next, by the definition of $\Kcal$, for all $t$, $\| w_t - u \| \leq r_2 = 4\theta$. By Lemma~\ref{lem:dist-angle}, this implies that $\theta(w_t, u) \leq \pi \cdot 4 \theta \leq 16\theta$. Moreover, by the fact that $\theta \leq \frac{\pi}{32}$, for all $t$, $\theta(w_t, u) \leq \frac{\pi}{2}$.

Now, applying Corollary~\ref{cor:refine-local}, we have that with probability $1-\delta'$, the $\cbr{w_t}_{t=1}^T$ generated by Algorithm~\ref{alg:refine} are such that
\begin{equation*}
\frac{1}{T} \sum_{t=1}^T f_{u,b}(w_t) \leq \frac{\theta}{50 \cdot 3^4 \cdot 2^{33}}.
\end{equation*}
Define $A = \cbr{t \in [T]: f_{u, b}(w_t) \geq \frac{\theta}{5 \cdot 3^4 \cdot 2^{21}} }$.
As $\frac{1}{T} \sum_{t=1}^T f_{u,b}(w_t) \geq \frac{\theta}{5 \cdot 3^4 \cdot 2^{21}} \cdot \frac{1}{T} \sum_{t=1}^T \ind{t \in A} = \frac{\theta}{5 \cdot 3^4 \cdot 2^{21}} \frac{\abs{A}}{T}$,
we have $\frac{\abs{A}}{T} \leq \frac{5 \cdot 3^4 \cdot 2^{21}}{50 \cdot 3^4 \cdot 2^{33}} = \frac{1}{10 \cdot 2^{12}}$.
Therefore, $\frac{\abs{\bar{A}}}{T} \geq 1 - \frac{1}{10 \cdot 2^{12}}$, and for all $t \in \bar{A}$ we have $f_{u, b}(w_t) \leq \frac{\theta}{50 \cdot 3^4 \cdot 2^{21}}$; by Claim~\ref{claim:f-angle} above, we have $\theta(w_t, u) \leq \frac{\theta}{5}$ for these $t$.

Using the fact that for all $t$ in $A$, $\theta(w_t, u) \leq 16\theta$, and the fact that for all $t$ in $\bar{A}$, $\theta(w_t, u) \leq \frac{\theta}{5}$, we have:
\begin{eqnarray*}
\frac1T \sum_{t=1}^T \cos\theta(w_t, u)
&\geq& \cos\frac{\theta}{5} \cdot \del{1 - \frac{1}{20 \cdot 2^{12}}} + \cos(16\theta) \cdot \frac{1}{20 \cdot 2^{12}} \\
&\geq& \del{ 1 - \frac{\theta^2}{40} } \del{1 - \frac{1}{20 \cdot 2^{12}}} + \del{1 - \frac{(16\theta)^2}{2}} \frac{1}{20 \cdot 2^{12}} \\
&\geq& 1 - \frac{\theta^2}{40} - \frac{\theta^2}{40} \; = \; 1 - \frac{\theta^2}{20} \; \geq \; \cos\frac{\theta}{2}.
 \end{eqnarray*}
where the second inequality uses item~\ref{item:0-lb} of Lemma~\ref{lem:cos-quad}, the third inequality is by algebra, and the last inequality uses item~\ref{item:0-ub} of Lemma~\ref{lem:cos-quad}.

The above inequality, in combination with Lemma~\ref{lem:avg-angle} yields the following guarantee for $\tilde{w}$:
\begin{equation*}
\cos \theta(\tilde{w}, u) \geq \frac1T\sum_{t=1}^T \cos\theta(w_t, u) \geq \cos \frac{\theta}{2}.
\end{equation*}
This implies that $\theta(\tilde{v}', u) \leq \frac{\theta}{2}$ since we set $\tilde{v}' = \tilde{w}$.
\end{proof}

\subsection{Proof of Lemma~\ref{lem:omd}}
Throughout this section, we define a filtration $\cbr{\Fcal_t}_{t=0}^{T}$ as follows: $\Fcal_0 = \sigma(w_1)$,
\begin{equation*}
\Fcal_t = \sigma(w_1, x_1, y_1, \ldots, w_t, x_t, y_t, w_{t+1}),
\end{equation*}
for all $t \in [T]$. As a shorthand, we write $\EE_{t-1}[\cdot]$ for $\EE\sbr{\cdot \mid \Fcal_{t-1}}$.

\begin{proof}[Proof of Lemma~\ref{lem:omd}]
From standard analysis of online mirror descent~\cite[see e.g.][Theorem 6.8]{orabona2019modern} with step size $\alpha$, constraint set $\Kcal$ and regularizer $\Phi(w) = \frac{1}{2(p-1)}\| w - w_1 \|_p^2$, we have that for every $u'$ in $\Kcal$,
\begin{equation*}
\alpha \cdot \sbr{ \sum_{t=1}^T \inner{w_t}{g_t} + \sum_{t=1}^T \inner{-u'}{g_t} } \leq D_{\Phi}(u', w_1) - D_{\Phi}(u, w_{T+1}) +  \sum_{t=1}^T \alpha^2 \| g_t \|_q^2.
\end{equation*}

Let $u' = u$ in the above inequality, drop the negative term on the right hand side, and observe that $\|g_t\|_q \leq 2 \| g_t \|_\infty \leq 2 \| x_t \|_\infty$ (see Lemma~\ref{lem:infty-q}), we have
\begin{equation*}
\alpha \cdot \sbr{ \sum_{t=1}^T \inner{w_t}{g_t} + \sum_{t=1}^T \inner{-u}{g_t} } \leq D_{\Phi}(u, w_1) +  \sum_{t=1}^T 4 \alpha^2 \| x_t \|_\infty^2.
\end{equation*}

Moving the first term to the right hand side, and divide both sides by $\alpha$, we get:

\begin{equation}
\sum_{t=1}^T \inner{-u}{g_t}
\leq \frac{D_{\Phi}(u, w_1)}{\alpha} + \sum_{t=1}^T \inner{-w_t}{g_t} + 4 \alpha \sum_{t=1}^T \| x_t \|_\infty^2.
\label{eqn:omd-grt}
\end{equation}

Let us look at each of the terms closely. First, we can easily upper bound $D_{\Phi}(u, w_1)$ by assumption:
\begin{equation}
D_{\Phi}(u, w_1) = \frac{\| u - w_1 \|_p^2}{2(p-1)} \leq \frac{\ln (8d) - 1}{2} r_1^2 \leq \frac{r_1^2 \ln (8d)}{2}.
\label{eqn:bregman-bound}
\end{equation}
where the first inequality uses the observation that as $p \geq 1$, $\| u - w_1 \|_p^2 \leq \| u - w_1 \|_1^2 \leq r_1^2$.


Let $W_t \defeq \inner{-w_t}{g_t}$. First, example $x_t$ is sampled from region $B_{\hat{w}_t, b}$, $\abs{\inner{\hat{w}_t}{x_t}} \leq b$. Moreover, by the assumption that $\Kcal \subset \cbr{w: \| w \|_2 \leq 1}$, we have $\| w_t \|_2 \leq 1$, implying that $\abs{\inner{w_t}{x_t}} \leq b$.
Therefore, $\abs{W_t} = \abs{\frac12 y_t - (\frac12 - \eta) \hat{y}_t} \abs{\inner{w_t}{x_t}} \leq b$. Consequently,
\begin{equation}
\sum_{t=1}^T W_t \leq T \cdot b.
\label{eqn:w-bound}
\end{equation}


Define $U_t \defeq \inner{-u}{g_t}$. By Lemma~\ref{lem:update-distance},
$\EE_{t-1} U_t \geq (1 - 2\eta) f_{u,b}(w_t)$.
Moreover, Lemma~\ref{lem:dev-u} implies that there is a numerical constant $c_1 > 0$, such that with probability $1-\delta / 3$:
$\abs{\sum_{t=1}^T U_t - \EE_{t-1} U_t} \leq c_1 (b + r_2) \del{\sqrt{T \ln\frac1\delta}  + \ln\frac{1}{\delta}}$. Consequently,
\begin{equation}
\sum_{t=1}^T (1-2\eta) f_{u,b}(w_t)
\leq \sum_{t=1}^T \EE_{t-1} U_t
\leq \sum_{t=1}^T U_t + c_1 (b + r_2) \del{\sqrt{T \ln\frac1\delta} + \ln\frac{1}{\delta}}.
\label{eqn:u-conc}
\end{equation}

Moreover, by Lemma~\ref{lem:infty-norm-band}, there exists a constant $c_2 > 0$, such that
with probability $1-\delta / 3$,
\begin{equation}
\sum_{t=1}^T \| x_t \|_\infty^2 \leq c_2 T \cdot \del{\ln \frac{T d}{\delta b}}^2.
\label{eqn:norm-bound}
\end{equation}

Combining Equations~\eqref{eqn:omd-grt},~\eqref{eqn:bregman-bound},~\eqref{eqn:w-bound},~\eqref{eqn:u-conc} and~\eqref{eqn:norm-bound}, along with union bound, we get that there exists a constant $c_3 > 0$, such that with probability $1-\delta$:
\begin{equation*}
(1 - 2\eta) \sum_{t=1}^T f_{u,b}(w_t) \leq c_3 \del{\alpha T \del{\ln \frac{T d}{\delta b}}^2 + \frac{r_1^2 \ln d }{\alpha} + b T + (b + r_2) \del{\sqrt{T \ln\frac{1}{\delta}} + \ln\frac1\delta}}.
\end{equation*}
The theorem follows by dividing both sides by $(1-2\eta)T$.
\end{proof}




\begin{lemma}\label{lem:dev-u}
Recall that $U_t = \inner{u}{-g_t}$.
There is a numerical constant $c$ such that the following holds. We have that with probability $1-\delta$,
\begin{equation}
\abs{\sum_{t=1}^T \del{U_t - \EE_{t-1} U_t}} \leq c (b+r_2) \del{\sqrt{T \ln\frac{1}{\delta}} + \ln\frac1\delta}.
\end{equation}
\end{lemma}
\begin{proof}
By item~\ref{item:ell-2} of the premise of Lemma~\ref{lem:omd}, along with the fact that $w_t \in \Kcal$, $\| u - w_t \| \leq r_2$, we hence have $\| u - \hat{w}_t \| \leq 2r_2$ using
Lemma~\ref{lem:normalize}. Therefore, Lemma~\ref{lem:abl-tail} implies the existence of constants $\beta$ and $\beta'$ such that for all $a \geq 0$,
\begin{equation*}
\PP_{x_t \sim D_{\hat{w}_t, b}} \del{ \abs{u \cdot x_t} \geq a } \leq \beta \exp\del{-\beta' \frac{a}{r_2+b}}.
\end{equation*}

Let $M_t = (\frac12 y_t - (\frac12 - \eta) \hat{y}_t)$. Observe that $\abs{M_t} \leq 1$. Therefore, $U_t = \inner{u}{g_t} = M_t u \cdot x_t$ has the exact same tail probability bound, i.e.
\begin{equation*}
\PP_{x_t \sim D_{\hat{w}_t, b}} \del{ \abs{U_t} \geq a } \leq \beta \exp\del{-\beta' \frac{a}{r_2+b}}.
\end{equation*}

The lemma now follows from Lemma~\ref{lem:subexp-tail-hoeff} in Appendix~\ref{sec:conc} with the setting of $Z_t = U_t$.
\end{proof}

Lemma~\ref{lem:dev-u} relies on the following useful lemma from~\cite{awasthi2017power}.
\begin{lemma}\label{lem:abl-tail}
There exist numerical constants $\beta$ and $\beta'$ such that for any isotropic log-concave distribution $D_X$ over $\RR^d$, any unit vector $\hat{w}$ in $\RR^d$ and $u \in \RR^d$ with
$\| u \|_2 \leq 1$, $\| u - \hat{w} \| \leq r$, any scalar $b$ in $[0,1]$, the following holds for all $a \geq 0$:
\begin{equation*}
\PP_{x \sim D_{\hat{w}, b}} \del{ \abs{u \cdot x} \geq a } \leq \beta \exp\del{-\beta' \frac{a}{r+b}}.
\end{equation*}
\end{lemma}
\begin{proof}
Using Lemma 3.3 of~\cite{awasthi2017power} with $C = 1$, we have that there exists numerical constants $c_0, c_0' > 0$, such that for any $K \geq 4$,
\begin{equation*}
\PP_{x \sim D_{w, b}} \del{ \abs{u \cdot x} \geq K \sqrt{r^2 + b^2}}
\leq c \exp\del{-c_0' K \sqrt{1 + \frac{b^2}{r^2}}} \leq c_0 \exp\del{-c_0' K}.
\end{equation*}
Therefore, for every $a \geq 4(r+b) \geq 4 \sqrt{r^2 + b^2}$,
\begin{equation*}
\PP_{x \sim D_{w, b}} \del{ \abs{u \cdot x} \geq a}
\leq c_0 \exp\del{-c_0' \frac{a}{\sqrt{r^2 + b^2}}}
\leq c_0 \exp\del{-c_0' \frac{a}{(r+b)}}.
\end{equation*}
In addition, for every $a < 4(r+b)$,
$\PP_{x \sim D_{w, b}} \del{ \abs{u \cdot x} \geq a} \leq 1$ trivially holds, in which case,
\begin{equation*}
\PP_{x \sim D_{w, b}} \del{ \abs{u \cdot x} \geq a} \leq 1 \leq \exp\del{4 c_0'} \exp\del{-c_0' \frac{a}{(r+b)}}.
\end{equation*}

Therefore, we can find new numerical constants $\beta = \max(c_0, \exp\del{4 c_0'})$ and $\beta' = c_0'$, such that
\begin{equation*}
\PP_{x \sim D_{w, b}} \del{ \abs{u \cdot x} \geq a } \leq \beta \exp\del{-\beta' \frac{a}{r+b}}
\end{equation*}
holds.
\end{proof}

The lemma below provides a coarse bound on the last term in the regret guarantee~\eqref{eqn:omd-grt}.
\begin{lemma}\label{lem:infty-norm-band}
With probability $1-\delta$,
$\sum_{t=1}^T \| x_t \|_\infty^2 \leq T \cdot \del{17 + \ln \frac{T d}{\delta b}}^2$.
\end{lemma}
\begin{proof}
Given $x \in \RR^d$ and $j \in [d]$, let $x^{(j)}$ be the $j$-th coordinate of $x$.
As $D_X$ is isotropic log-concave, for $x \sim D_X$, from Lemma~\ref{lem:log-concave-tail} we have that for all coordinates $j$ in $\cbr{1,\ldots, d}$ and every $a > 0$,
\begin{equation}
\PP_{x \sim D_X} \del{ \abs{x^{(j)}} \geq a } \leq \exp(-a+1).
\end{equation}
Therefore, using union bound, we have
\begin{equation*}
\PP_{x \sim D_X} \del{ \| x \|_\infty \geq a } \leq  d \exp(-a+1).
\end{equation*}
In addition, as $b \in [0,\frac{\pi}{72}] \subset [0, \frac19]$, we have by Lemma~\ref{lem:ilc-density-lb}, $\PP_{x \sim D_X}\del{x \in R_{\hat{w}, b}} \geq \frac{b}{2^{16}}$.

Now, by the simple fact that $\PP(A | B) \leq \frac{\PP(A)}{\PP(B)}$, we have that
\begin{align*}
\PP_{x \sim D_{\hat{w}, b}} \del{ \| x \|_\infty \geq a }
= \PP_{x \sim D_X} \del{ \| x \|_\infty \geq a \vert x \in R_{\hat{w}, b} }
&\leq  \frac{\PP_{x \sim D_X} \del{ \| x \|_\infty \geq a }}{\PP_{x \sim D_X}\del{x \in R_{\hat{w}, b}}}\\
& \leq \frac{2^{16} d}{b} \exp(-a+1).
\end{align*}
Therefore, taking $a = 17 + \ln \frac{T d}{\delta b}$ in the above inequality, we get that, the above event happens with probability at most $\frac{\delta}{T}$. In other words,
with probability $1-\frac{\delta}{T}$,
\begin{equation}
\| x_t \|_\infty \leq 17 + \ln \frac{T d}{\delta b}.
\label{eqn:l-inf-bound}
\end{equation}

Thus, taking a union bound, we get that with probability $1-\delta$, for every $t$, Equation~\eqref{eqn:l-inf-bound} holds. As a result, $\sum_{t=1}^T \| x_t \|_\infty^2 \leq T \cdot \del{17 + \ln \frac{T d}{\delta b}}^2$.
\end{proof}

\section{Analysis of Initialization: Proof of Theorem~\ref{thm:init}}
\label{sec:init}

\subsection{Obtaining a halfspace constraint on $u$}




Before going into the proof of Theorem~\ref{thm:init}, we introduce a few notations.
Throughout this section, we use $\EE$ to denote $\EE_{(x,y) \sim D}$ as a shorthand.
Denote by $\bar{w} \defeq \EE \sbr{x y}$, and denote $\hat{\EE}$ as the empirical expectation over $(x_i, y_i)_{i=1}^m$; with this notation,
$w_{\avg} = \hat{\EE}\sbr{xy}$. Denote by $w_{\tilde{s}} = \calH_{\tilde{s}}(w_{\avg})$; in this notation, $w^\sharp = \frac{w_{\tilde{s}}}{\| w_{\tilde{s}} \|}$.

\begin{lemma}\label{lem:w-sharp}
If Algorithm~\ref{alg:init} is run with hard-thresholding parameter $\tilde{s} = 81 \cdot 2^{38} \cdot \frac{s}{(1-2\eta)^2}$, number of labeled examples $m = 81 \cdot 2^{51} \cdot \frac{s}{(1-2\eta)^2} \ln\frac{8d}{\delta'}$, then with probability $1-\delta'/2$, the unit vector $w^\sharp$ obtained at line~\ref{line:init-ht} is such that
\begin{equation}
\inner{w^\sharp}{u} \geq \frac{(1-2\eta)}{9 \cdot 2^{19}}.
\end{equation}
\end{lemma}
\begin{proof}
First, Lemma~\ref{lem:corr} below implies that
\begin{equation}
\inner{ \bar{w} }{u} \geq  \frac{\del{1-2\eta}}{9 \cdot 2^{16}}.
\end{equation}

Moreover, as $u$ is a unit vector, and $D_X$ is isotropic log-concave, $\inner{u}{x}$ comes from a one-dimensional isotropic log-concave distribution. In addition, $y$ is a random variable that takes values in $\cbr{\pm 1}$. Therefore, by Lemma~\ref{lem:log-concave-subexp}, $y\inner{u}{x}$ is $(32, 16)$-subexponential. Lemma~\ref{lem:subexp-tail}, in allusion to the choice of $m$, implies that with probability $1-\delta'/4$,
\begin{equation}
\abs{ \frac{1}{m}\sum_{i=1}^{m} \sbr{y_i \inner{u}{x_i}} - \EE \sbr{y \inner{u}{x}} }
\leq 32 \sqrt{\frac{2 \ln\frac 8 \delta}{m}} + 32 \frac{\ln\frac 8 \delta}{m}
\leq \frac{(1-2\eta)}{9 \cdot 2^{17}}.
\end{equation}
Thus,
\begin{equation}
\inner{w_{\avg}}{u} = \inner{\frac{1}{m}\sum_{i=1}^{m}y_i x_i }{u} \geq \frac{(1-2\eta)}{9 \cdot 2^{16}} - \frac{(1-2\eta)}{9 \cdot 2^{17}}  \geq \frac{(1-2\eta)}{9 \cdot 2^{17}}.
\label{eqn:ip-lb}
\end{equation}

Now, consider $w_{\tilde{s}} = \calH_{\tilde{s}}(w_{\avg})$. By Lemma~\ref{lem:subset-norm} shown below, with the choice of $m$, we have that with probability $1-\delta'/4$, $\| w_{\tilde{s}} \|_2 \leq 2$.
Hence, by union bound, with probability $1-\delta'/2$, both Equation~\eqref{eqn:ip-lb} and $\| w_{\tilde{s}} \|_2 \leq 2$ hold.

In this event, Lemma~\ref{lem:ht-ip-2} (also shown below), in combination with the fact that $\tilde{s} = \frac{81 \cdot 2^{38} s}{(1-2\eta)^2}$, implies that
\begin{equation}
\inner{w_{\tilde{s}}}{u} \geq \inner{w_{\avg}}{u} - \sqrt{\frac{s}{\tilde{s}}} \| w_{\tilde{s}} \| \geq
\frac{(1-2\eta)}{9 \cdot 2^{17}} - \frac{(1-2\eta)}{9 \cdot 2^{19}} \cdot 2 =
\frac{(1-2\eta)}{9 \cdot 2^{18}}.
\end{equation}
By the fact that $w^\sharp = \frac{w_{\tilde{s}}}{\| w_{\tilde{s}} \|}$ and using again $\| w_{\tilde{s}} \| \leq 2$, we have
\begin{equation*}
\inner{w^\sharp}{u} \geq \frac12 \inner{w_{\tilde{s}}}{u} \geq \frac{(1-2\eta)}{9 \cdot 2^{19}},
\end{equation*}
which is the desired result.
\end{proof}


\begin{lemma}\label{lem:corr}
Suppose that Assumption~\ref{a:bn} is satisfied. Then
\begin{equation*}
\inner{\bar{w}}{u} \geq \frac{(1-2\eta)}{9 \cdot 2^{16}}.
\end{equation*}
\end{lemma}
\begin{proof}
Recall that $\bar{w} = \EE[xy]$. We have
\begin{eqnarray*}
\inner{\bar{w}}{u} &=& \EE[y (u \cdot x)] \\
&=& \EE\sbr{\EE[y \mid x] (u \cdot x)} \\
&\geq& (1 - 2\eta) \EE \sbr{ \abs{u \cdot x} } \geq \frac{(1-2\eta)}{9 \cdot 2^{16}},
\end{eqnarray*}
where the first equality is by the linearity of inner product and expectation, the second equality is by the tower property of conditional expectation.
The first inequality uses Lemma~\ref{lem:u-ip-x}.
For the last inequality, we use the fact that $z = u \cdot x$ can be seen as drawn from a one-dimensional isotropic log-concave distribution with density $f_Z$, along with Lemma~\ref{lem:ilc-density-lb} with $d = 1$ with states that for every $z \in [0, 1/9]$, $f_Z(z) \geq 2^{-16}$, making $\EE_{z \sim f_Z}[{\abs{z}}]$ bounded from below by $\frac{1}{9 \cdot 2^{16}}$.
\end{proof}

The following lemma is inspired by Lemma 12 of~\cite{yuan2013truncated}.
\begin{lemma}\label{lem:ht-ip-2}
For any vector $a$ and any $s$-sparse unit vector $u$, we have
\begin{equation*}
\abs{ \inner{\calH_{\tilde{s}}(a)}{u} - \inner{a}{u} } \leq \sqrt{\frac{s}{{\tilde{s}}}} \twonorm{ \calH_{\tilde{s}}(a) }.
\end{equation*}
\end{lemma}
\begin{proof}
Let $\Omega$ be the support of $\calH_{\tilde{s}}(a)$, and $\Omega'$ be the support of $u$.
Given any vector $v$, denote by $v_1$ (resp. $v_2$, $v_3$) the vector obtained by zeroing out all elements outside $\Omega \setminus \Omega'$ (resp. $\Omega \cap \Omega'$, $\Omega' \setminus \Omega$) from $v$. With this notation, it can be seen that $\calH_{\tilde{s}}(a) = a_2 + a_3$,
$\inner{\calH_k(a)}{u} = \inner{a_2}{u_2}$, $\inner{a}{u} = \inner{a_2}{u_2} + \inner{a_3}{u_3}$. Thus, it suffices to prove that $\abs{\inner{a_3}{u_3}} \leq \sqrt{\frac{s}{{\tilde{s}}}} \twonorm{ \calH_{\tilde{s}}(a) }$.


First, this holds in the trivial case that $a_3$ is a zero-vector. Now suppose that $a_3$ is non-zero. By the definition of $\calH_{\tilde{s}}$, this implies that all the elements of $\calH_{\tilde{s}}(a)$ is non-zero, and hence $\zeronorm{\calH_{\tilde{s}}(a)} = {\tilde{s}}$. In addition, every element of $a_3$ has absolute value smaller than that of $\calH_{\tilde{s}}(a)$. Consequently, the average squared element of $a_3$ is larger than that of $\calH_{\tilde{s}}(a)$, namely
\begin{equation}
\frac{\twonorm{a_3}^2}{\zeronorm{a_3}}
\leq \frac{\twonorm{\calH_{\tilde{s}}(a)}^2}{\zeronorm{\calH_{\tilde{s}}(a)}}.
\end{equation}
Since $\zeronorm{a_3} = \abs{\Omega' \backslash \Omega} \leq \abs{\Omega'} = s$, and $\zeronorm{\calH_{\tilde{s}}(a)} = {\tilde{s}}$, we obtain $\twonorm{a_3} \leq \sqrt{\frac{s}{{\tilde{s}}}} \twonorm{a_1}$. The result follows by observing that $\abs{\inner{a_3}{u_3}} \leq \twonorm{a_3} \cdot \twonorm{u_3} \leq \twonorm{a_3}$ where the first inequality is by Cauchy-Schwarz and the second one is from the premise that $\twonorm{u}=1$.
\end{proof}

Recall that $w_{\avg} = \hat{\EE} \sbr{xy}$ is the vector obtained by empirical average all $x_i y_i$'s. In the lemma below, we argue that
the $\ell_2$ norm of $w_{\tilde{s}} = \calH_{\tilde{s}}(w_{\avg})$ is small. As a matter of fact, we show a stronger result that, keeping any $\tilde{s}$ elements of vector $w$ (and zeroing out the rest) makes the resulting vector have a small norm.
\begin{lemma}
Suppose $\tilde{s} \in [d]$ is a natural number.
With probability $1-\delta' / 4$ over the draw of $m = 2^{13} \cdot \tilde{s} \ln\frac{8d}{\delta'}$ examples, the following holds:
For any subset $\Omega \subset [d]$ of size $\tilde{s}$, we have that
$\| (w_{\avg})_{\Omega} \| \leq 2$, where $(w_{\avg})_{\Omega}$ is obtained by zeroing out all but the elements in $\Omega$.
\label{lem:subset-norm}
\end{lemma}
\begin{proof}
We prove the lemma in two steps: first, we show that $\bar{w} = \EE\sbr{xy}$ must have a small $\ell_2$ norm -- specifically, this implies that $\| \bar{w}_{\Omega} \|_2$ is small;
second, we show that $\bar{w}$ and $w_{\avg}$ are close to each other entrywise. Then we combine these two observations
to show that $(w_{\avg})_\Omega$ has a small $\ell_2$ norm. Write the vector $\bar{w} = (\bar{w}^{(1)}, \bar{w}^{(2)}, \dots, \bar{w}^{(d)})$ and the vector $x = (x^{(1)}, x^{(2)}, \dots, x^{(d)})$.

For the first step, by Lemma~\ref{lem:corr-bound} shown below, we have
\begin{equation}
\sum_{i \in \Omega} \(\bar{w}^{(j)}\)^2 \leq \sum_{j=1}^d \(\bar{w}^{(j)}\)^2 = \sum_{j=1}^d (\EE\[ x^{(j)} y \])^2 \leq 1.
\end{equation}
For the second step, we know that as $x^{(j)}$ is drawn from an isotropic log-concave and $y$ take values in $\cbr{\pm 1}$, by Lemma~\ref{lem:log-concave-subexp} in Appendix~\ref{sec:conc}, $x^{(j)} y$ is $(32, 16)$-subexponential.
Therefore, by Lemma~\ref{lem:subexp-tail}, along with union bound, we have that with probability $1-\delta$, for all  coordinates $j$ in $[d]$,
\begin{equation}\label{eq:tmp}
\abs{w_{\avg}^{(j)} - \bar{w}^{(j)}} = \abs{ \hat{\EE}[x^{(j)} y] - \EE[x^{(j)} y]} \leq 32 \sqrt{2 \frac{\ln\frac {2d} \delta}{m}} + 32 \frac{\ln\frac {2d} \delta}{m} \leq \frac{1}{\sqrt{\tilde{s}}},
\end{equation}
where the last inequality is from our setting of $m$.

The above two items together imply that,
\begin{equation}
\sum_{j \in \Omega} \(w_{\avg}^{(j)}\)^2 \leq \sum_{j \in \Omega} 2 \(\bar{w}^{(j)}\)^2 + 2 \(w_{\avg}^{(j)} - \bar{w}^{(j)}\)^2 \leq 2 + 2 \frac{\tilde{s}}{\tilde{s}} \leq 4.
\end{equation}
The lemma is concluded by recognizing that the left hand side is $\|(w_{\avg})_{\Omega} \|^2$ and by setting $\delta = \delta'/4$ in \eqref{eq:tmp}.
\end{proof}

\begin{lemma}\label{lem:corr-bound}
Given a vector $x \in \Xcal$, we write $x = (x^{(1)}, x^{(2)}, \dots, x^{(d)})$. We have
\begin{equation}
\sum_{j=1}^d \big(\EE\big[ x^{(j)} y \big] \big)^2 \leq 1.
\end{equation}
\end{lemma}
\begin{proof}
Denote by function $\zeta(x) \defeq \EE[y | x]$. As $y \in \cbr{\pm 1}$, we have that for every $x$, $\zeta(x) \in [-1, +1]$.
In this notation, by the tower property of expectation, $\EE\[ x^{(j)} y \] =
\EE \sbr{x^{(j)} \zeta(x)}$.

For $f, g$ in $L^2(D_X)$, we denote by $\inner{f}{g}_{L^2(D_X)} = \EE_{x \sim D_X}\sbr{f(x) g(x)}$ their inner product in $L^2(D_X)$.
As $D_X$ is isotropic,
\begin{equation*}
\inner{x^{(j)}}{x^{(j)}}_{L^2(D_X)} =
\EE_{x \sim D_X}\sbr{x^{(j)} x^{(j)}} =
\begin{cases}
1,& i = j, \\
0,& i \neq j.
\end{cases}
\end{equation*}
Therefore, $x^{(1)}, \ldots, x^{(d)}$ is a set of orthonormal functions in $L^2(D_X)$. This implies
\begin{equation}
\sum_{j=1}^d \del{\EE\[ x^{(j)} \zeta(x) \]}^2 = \sum_{j=1}^d \inner{\zeta}{x^{(j)}}_{L^2(D_X)}^2 \leq \inner{\zeta}{\zeta}_{L^2(D_X)} \leq 1.
\end{equation}
where the equality is from the definition of $\inner{f}{g}_{L^2(D_X)}$, the first inequality is from Bessel's inequality, and the second inequality uses the fact that $\zeta(x)^2 \in [0, 1]$ and $D_X$ is a probability measure. This completes the proof.
\end{proof}

\subsection{Obtaining a vector that has a small angle with $u$}

One technical challenge in directly applying the same analysis of Theorem~\ref{thm:refine} to the initialization phase is that, some of the $w_t$'s obtained may have large obtuse angles with $u$ (e.g. $\theta(w_t,u)$ is close to $\pi$), making their corresponding $f_{u,b}(w_t)$ value small. To prevent this undesirable behavior, Algorithm~\ref{alg:init} add a linear constraint $\inner{w}{w^\sharp} \geq \frac{(1-2\eta)}{9 \cdot 2^{19}}$ on the set $\Kcal$ when applying \refine, which ensures that all vectors in $\Kcal$ will have angle with $u$ bounded away from $\pi$. The lemma below formalizes this intuition.

Recall that Algorithm~\ref{alg:init} sets $\Kcal = \cbr{w: \| w \|_2 \leq 1, \| w \|_1 \leq \sqrt{s}, \inner{w}{w^\sharp} \geq \frac{(1-2\eta)}{9 \cdot 2^{19}}}$.

\begin{lemma}\label{lem:angle-not-flat}
For any two vectors $w_1, w_2 \in \Kcal$, the angle between them, $\theta(w_1, w_2)$, is such that
\begin{equation*}
\theta(w_1, w_2) \leq \pi - \frac{\del{1-2\eta}}{9 \cdot 2^{19}}.
\end{equation*}
\end{lemma}
\begin{proof}
First, by the definition of $\Kcal$, for $w_1, w_2$ in $\Kcal$, we have
$\inner{w_i}{w^\sharp} \geq \frac{\del{1-2\eta}}{9 \cdot 2^{19}}$
for $i = 1,2$.
In addition, by the definition of $\Kcal$, both $w_1$ and $w_2$ have norms at most 1. This implies that their normalized version, $\hat{w}_1$ and $\hat{w}_2$, satisfies,
$\inner{\hat{w}_i}{w^\sharp} \geq \frac{\del{1-2\eta}}{9 \cdot 2^{19}}$
for $i = 1,2$.

For $i = 1,2$, let $\hat{w}_i = \hat{w}_{i,\parallel} + \hat{w}_{i,\perp}$ be an orthogonal decomposition, where $\hat{w}_{i,\parallel}$ (resp. $\hat{w}_{i,\perp}$) denotes the component of $\hat{w}_i$ parallel to (resp. orthogonal to) $w^\sharp$.
As $\| \hat{w}_i \| \leq 1$,  we have that $\| \hat{w}_{i,\perp} \| \leq 1$, implying that $\abs{\inner{\hat{w}_{1,\perp}}{\hat{w}_{2,\perp}}} \leq \| \hat{w}_{1,\perp} \| \cdot \| \hat{w}_{2,\perp} \| \leq 1$.
In addition,
$\inner{\hat{w}_{1,\parallel}}{\hat{w}_{2,\parallel}} = \inner{\hat{w}_1}{w^\sharp} \cdot \inner{\hat{w}_2}{w^\sharp} \geq \del{\frac{\del{1-2\eta}}{9 \cdot 2^{19}}}^2$.
Therefore,
\begin{equation*}
\cos\theta(w_1, w_2) = \inner{\hat{w}_1}{\hat{w}_2} = \inner{\hat{w}_{1,\parallel}}{\hat{w}_{2,\parallel}} + \inner{\hat{w}_{1,\perp}}{\hat{w}_{2,\perp}} \geq -1 + \del{\frac{\del{1-2\eta}}{9 \cdot 2^{19}}}^2.
\end{equation*}
By item~\ref{item:pi-ub} of Lemma~\ref{lem:cos-quad}, we get that
\begin{equation*}
-1 + \frac{1}{2} \del{\theta(w_1, w_2) - \pi}^2 \geq -1 + \del{\frac{\del{1-2\eta}}{9 \cdot 2^{19}}}^2,
\end{equation*}
The above inequality, in combination with the basic fact that $\theta(w_1, w_2) \in [0,\pi]$, implies that $\theta(w_1, w_2) \leq \pi - \frac{\del{1-2\eta}}{9 \cdot 2^{19}}$.
\end{proof}




The following lemma is the main result of this subsection, which shows that by using the new constraint set $\Kcal$ in Algorithm~\ref{alg:init}, \refine obtains a vector with constant angle with $u$ with $\tilde{\order}\del{\frac{s}{(1-2\eta)^4}}$ labels.

\begin{lemma}\label{lem:halfspace-init}
Suppose we are given a unit vector $w^\sharp$ such that $\inner{w^\sharp}{u} \geq \frac{(1-2\eta)}{9 \cdot 2^{19}}$.
If Algorithm~\ref{alg:refine} is run with initialization $w_1 = 0$, bandwidth $b = \Theta\del{(1-2\eta)^2}$, step size $\alpha = \Theta\del{(1-2\eta)^2 /\del{\ln\frac{d}{\delta' (1-2\eta)}}^2}$, constraint set $\Kcal = \cbr{w: \| w \|_2 \leq 1, \inner{w}{w^\sharp} \geq \frac{(1-2\eta)}{9 \cdot 2^{19}}}$, regularizer $\Phi(w) = \frac{1}{2(p-1)}\| w\|_p^2$, number of iterations $T = \order\del{\frac{s }{(1-2\eta)^4}\del{\ln\frac{d}{\delta' (1-2\eta)}}^3}$, then with probability $1-\frac{\delta'}{2}$, it returns a vector $\tilde{v}_0$ such that  $\theta(\tilde{v}_0, u) \leq \frac \pi {32}$.
\end{lemma}
\begin{proof}
We first check the premises of Lemma~\ref{lem:omd} with the chosen $w_1 \in \Kcal$, constraint set
\begin{equation*}
\Kcal = \cbr{w: \| w \|_2 \leq 1, \| w \|_1 \leq \sqrt{s}, \inner{w}{w^\sharp} \geq \frac{(1-2\eta)}{9 \cdot 2^{19}}},
\end{equation*}
$r_1 = 2\sqrt{s}$, $r_2 = 2$:
\begin{enumerate}
\item Observe that $\| u \|_1 \leq \sqrt{\| u \|_0} \| u \|_2 \leq \sqrt{s}$; in addition, by the definition
of $\Kcal$, $\| w_1 \|_1 \leq \sqrt{s}$. Therefore, $\| w_1 - u \|_1 \leq \| u \|_1 + \| w_1 \|_1 \leq 2\sqrt{s} = r_1$;
\item $w_1$ is in $\Kcal$ by definition; for $u$, we have $\| u \|_2 = 1$ by definition; $\| u \|_1 \leq \sqrt{s}$ by the argument above; $\inner{u}{w^\sharp} \geq \frac{(1-2\eta)}{9 \cdot 2^{19}}$. Therefore, $u$ is also in $\Kcal$.
\item For every $w$ in $\Kcal$, as $\| w \|_2 \leq 1$, we have $\| w - u \| \leq \| w \|_2 + \| u \|_2 = r_2$; in addition, by the definition of $\Kcal$, every $w$ in $\Kcal$ satisfies that $\| w \| \leq 1$.
\end{enumerate}

Therefore, applying Lemma~\ref{lem:omd}, we have that with probability $1-\frac{\delta'}{2}$,
\begin{equation*}
\frac{1}{T} \sum_{t=1}^T f_{u,b}(w_t) \leq c \cdot \del{\frac{\alpha \del{\ln \frac{2 T d}{\delta' b}}^2}{(1-2\eta)} + \frac{4 s \ln d}{\alpha (1-2\eta) T} + \frac{b}{(1-2\eta)} + \frac{(b + 2)}{(1-2\eta)} \del{\sqrt{\frac{\ln\frac{1}{\delta'}}T} + \frac{\ln\frac1{\delta'}}{T}}}.
\end{equation*}

Specifically, with the choice of $\alpha = \Theta\del{(1-2\eta)^2 /\del{\ln\frac{d}{\delta' (1-2\eta)}}^2}$, $b = \order\del{(1-2\eta)^2}$, $T = \order\del{\frac{s }{(1-2\eta)^4}\del{\ln\frac{d}{\delta' (1-2\eta)}}^3}$, we have that all four terms $\nicefrac{\alpha \del{\ln \frac{2 T d}{\delta' b}}^2}{(1-2\eta)}$, $\nicefrac{4 s \ln d}{\alpha (1-2\eta) T}$, $\nicefrac{b}{(1-2\eta)}$, $\nicefrac{(b + 2) \cdot \del{\sqrt{\frac{\ln\frac{1}{\delta'}}T} + \frac{\ln\frac1{\delta'}}{T}}}{(1-2\eta)}$ are all at most $\frac{(1-2\eta)}{c \cdot 5 \cdot 3^6 \cdot 2^{51}}$, implying that
\begin{equation*}
\frac{1}{T} \sum_{t=1}^T f_{u,b}(w_t) \leq 4c \cdot \frac{(1-2\eta)}{c \cdot 5 \cdot 3^6 \cdot 2^{51}} \leq \frac{(1-2\eta)}{5 \cdot 3^6 \cdot 2^{49}}.
\end{equation*}

Define $A = \cbr{t \in [T]: f_{u,b}(w_t) \geq \frac{(1-2\eta)}{3^6 \cdot 2^{40}} }$.
As $\frac{1}{T} \sum_{t=1}^T f_{u,b}(w_t) \geq \frac{(1-2\eta)}{3^6 \cdot 2^{40}} \cdot \frac{1}{T} \sum_{t=1}^T \ind{t \in A} = \frac{(1-2\eta)}{3^6 \cdot 2^{40}} \frac{\abs{A}}{T}$, we have $\frac{\abs{A}}{T} \leq \frac{1}{5 \cdot 2^9}$. Therefore, $\frac{\abs{\bar{A}}}{T} \geq 1 - \frac{1}{5 \cdot 2^9}$, and for every $t$ in $\bar{A}$, $w_t$ is such that $f_{u,b}(w_t) < \frac{(1-2\eta)}{3^6 \cdot 2^{40}}$.

We establish the following claim that characterizes the iterates $w_t$ where $t \in \bar{A}$.
\begin{claim}\label{claim:f-to-angle}
If $w \in \Kcal$ and $f_{u,b}(w) < \frac{(1-2\eta)}{3^6 \cdot 2^{40}}$, then $\theta(w, u) < \frac{(1-2\eta)}{9 \cdot 2^{19}}$.
\end{claim}
\begin{proof}
First, we show that it is impossible for $\theta(w, u) \geq \frac{\pi}{2}$.
By Lemma~\ref{lem:angle-not-flat}, for all $w$ in $\Kcal$, we have that
$\theta(w, u) \leq \pi - \frac{(1-2\eta)}{9 \cdot 2^{19}}$. By the choice of $b$, we know that $\theta(w, u) \leq \pi - 36 b$. By item~\ref{item:obtuse} of Lemma~\ref{lem:f-refined}, we have
\begin{equation}
f_{u,b}(w) \geq \frac{\pi - \theta(w, u)}{3^4 \cdot 2^{21}} \geq \frac{(1-2\eta)}{3^6 \cdot 2^{40}},
\end{equation}
which contradicts with the premise that $f_{u,b}(w) < \frac{(1-2\eta)}{3^6 \cdot 2^{40}}$.

Therefore, $\theta(w, u) \in [0, \frac \pi 2]$. We now conduct a case analysis.
\begin{enumerate}
\item If $\theta(w, u) \leq 36 b$, then by the definition of $b$, we automatically have $\theta(w, u) < \frac{(1-2\eta)}{9 \cdot 2^{19}}$.
\item Otherwise, $\theta(w, u) \in [36b, \frac \pi 2]$. In this case, by item~\ref{item:acute} of Lemma~\ref{lem:f-refined}, we have
\begin{equation*}
f_{u,b}(w) \geq \frac{\theta(w, u)}{3^4 \cdot 2^{21}}.
\end{equation*}
This inequality, in conjunction with the assumption that $f_{u,b}(w) < \frac{(1-2\eta)}{3^6 \cdot 2^{40}}$, implies that $\theta(w, u) \leq \frac{(1-2\eta)}{9 \cdot 2^{19}}$.
\end{enumerate}
In summary, in both cases, we have $\theta(w, u) \leq \frac{(1-2\eta)}{9 \cdot 2^{19}}$. This completes the proof.
\end{proof}

Claim~\ref{claim:f-to-angle} above implies that, for all $t$ in $\bar{A}$, $\theta(w_t, u) \leq \frac{(1-2\eta)}{9 \cdot 2^{19}} \leq \frac{\pi}{128}$. In addition, $\frac{\abs{\bar{A}}}{T} \geq 1 - \frac{1}{5 \cdot 2^9}$. Combining the above facts with the simple fact that $\cos\theta(w_t, w) \geq -1$ for all $t$ in $A$, we have:
\begin{eqnarray*}
 \frac{1}{T} \sum_{t=1}^T \cos \theta(w_t, u)  &\geq& \cos\frac{\pi}{128} \cdot \del{1 - \frac{1}{5 \cdot 2^9}} - 1 \cdot \del{\frac{1}{5 \cdot 2^9}}  \\
&\geq& \del{1 - \frac12 \del{\frac{\pi}{128}}^2} \cdot (1 - \frac{1}{5 \cdot 2^9}) - \frac{1}{5 \cdot 2^9} \\
&\geq& 1 - \frac15 \del{\frac{\pi}{32}}^2 \\
&\geq& \cos\frac{\pi}{32}
\end{eqnarray*}
where the first inequality is from the above conditions on $A$ and $\bar{A}$ we obtained; the second inequality uses item~\ref{item:0-lb} of Lemma~\ref{lem:cos-quad}; the third inequality is by algebra; the last inequality uses item~\ref{item:0-ub} of Lemma~\ref{lem:cos-quad}.

Combining the above result with Lemma~\ref{lem:avg-angle}, we have the following for $\tilde{v}_0 = \tilde{w}$:
\begin{equation*}
\cos \theta(\tilde{v}_0, u) = \cos \theta(\tilde{w}, u) \geq \frac1T\sum_{t=1}^T \cos \theta(w_t, u) \geq \cos \frac{\pi}{32}.
\end{equation*}
This implies that $\theta(\tilde{v}_0, u) \leq \frac{\pi}{32}$.
\end{proof}

Theorem~\ref{thm:init} is now a direct consequence of Lemmas~\ref{lem:w-sharp} and~\ref{lem:halfspace-init}.
\begin{proof}[Proof of Theorem~\ref{thm:init}]
First, by Lemma~\ref{lem:w-sharp}, we have that there exists an event $E_1$ that happens with probability $1-\delta'/2$, in which the unit vector $w^\sharp$ obtained is such that $\inner{w^\sharp}{u} \geq \frac{(1-2\eta)}{9 \cdot 2^{19}}$.
In addition, Lemma~\ref{lem:halfspace-init} states that there exists an event $E_2$ with probability $1-\delta'/2$, in which if $\inner{w^\sharp}{u} \geq \frac{(1-2\eta)}{9 \cdot 2^{19}}$, it returns $\tilde{v}_0$ such that $\theta(\tilde{v}_0, u) \leq \frac{\pi}{32}$. The theorem follows from considering the event $E_1 \cap E_2$, which happens with probability $1-\delta'$, in which $\tilde{v}_0$, the final output of Algorithm~\ref{alg:init}, satisfies that $\theta(\tilde{v}_0, u) \leq \frac{\pi}{32}$.
The total number of label queries made by Algorithm~\ref{alg:init} is:
\begin{equation*}
n = \order\del{\frac{s \ln d}{(1-2\eta)^2} + \frac{s}{(1-2\eta)^4}\del{\ln\frac{d}{\delta' (1-2\eta)}}^3}
= \order\del{\frac{s}{(1-2\eta)^4}\del{\ln\frac{d}{\delta' (1-2\eta)}}^3}.
\end{equation*}
\end{proof}

\section{The Structure of Function $f_{u,b}$}
\label{sec:f}
Recall that
\begin{equation*}
f_{u,b}(w) = \EE_{(x,y) \sim D_{\hat{w},b}} \sbr{\abs{u \cdot x} \ind{\sign{\inner{w}{x}} \neq \sign{\inner{u}{x}}}}.
\end{equation*}
Note that for all $w$, $f_{u,b}(w) \geq 0$. In this section, we show a few key properties of $f_{u,b}$, in that if $w$ has an acute angle with $u$, $f_{u,b}(w)$ behaves similar to the $\theta(w,u)$; if $w$ has an obtuse angle with $u$, $f_{u,b}(w)$ behaves similar to $\pi - \theta(w,u)$.

\begin{lemma}\label{lem:f-refined}
Suppose $w$ and $u$ are two unit vectors; in addition, suppose $b \leq \frac{\pi}{72}$. We have:
\begin{enumerate}
\item If $\theta(u, w) \in [36 b, \frac \pi 2]$, then $f_{u,b}(w) \geq \frac{ \theta(w,u) }{3^4 \cdot 2^{21}}$.
\label{item:acute}
\item If $\theta(u, w) \in [\frac\pi 2, \pi - 36 b]$, then $f_{u,b}(w) \geq \frac{ \pi - \theta(w,u) }{3^4 \cdot 2^{21}}$.
\label{item:obtuse}
\end{enumerate}
\end{lemma}
\begin{proof}
We prove the two items respectively.
\begin{enumerate}
\item For the first item, we denote by $\phi \defeq \theta(u, w)$.
Define region
\begin{equation*}
R_1 = \cbr{x: \inner{w}{x} \in [0, b], \inner{u}{x} \in \sbr{-\frac{\sin\phi}{36}, -\frac{\sin\phi}{18}}}.
\end{equation*}
It can be easily seen that $R_1$ is a subset of the disagreement region between $w$ and $u$. In other words,
\begin{equation*}
\ind{x \in R_1} \leq \ind{\sign{\inner{w}{x}} \neq \sign{\inner{u}{x}}}.
\end{equation*}

It suffices to show that, region $R_1$ has probability mass at least $\frac{b}{9 \cdot 2^{18}}$ wrt $D_X$. To see why it completes the proof, observe that
\begin{align*}
&\ \EE_{x \sim D_X} \sbr{\abs{u \cdot x} \ind{\abs{\inner{w}{x}} \leq b} \ind{\sign{\inner{w}{x}} \neq \sign{\inner{u}{x}}}} \\
\geq&\ \EE_{x \sim D_X} \sbr{\abs{u \cdot x} \ind{x \in R_1}} \\
\geq&\ \frac{\sin\phi}{36} \cdot \EE_{x \sim D_X} \ind{x \in R_1} \\
\geq&\ \frac{\phi}{72} \cdot \PP_{x \sim D_X}\del{x \in R_1}
\geq \frac{\phi \cdot b}{3^4 \cdot 2^{21}},
\end{align*}
where the first inequality uses the fact that $R_1$ is a subset of both $\cbr{x: \abs{\inner{w}{x}} \leq b}$ and $\cbr{x: \sign{\inner{w}{x}} \neq \sign{\inner{u}{x}}}$; the second inequality uses the fact that for all $x$ in $R_1$, $\abs{u \cdot x} \geq \frac{\sin \phi}{36}$; the third inequality uses the elementary fact that $\sin\phi \geq \frac{\phi}{2}$.

As $\PP_{x \sim D_X} \del{ \abs{\inner{w}{x}} \leq b } \leq b$ by Lemma~\ref{lem:ilc-density-ub}, this implies that
\begin{align*}
f_{u,b}(w) = \frac{\EE_{x \sim D_X} \sbr{\abs{u \cdot x} \ind{\abs{\inner{w}{x}} \leq b} \ind{\sign{\inner{w}{x}} \neq \sign{\inner{u}{x}}}}}{\PP_{x \sim D_X} \del{ \abs{\inner{w}{x}} \leq b }}
&\geq \frac{\phi \cdot b}{9 \cdot 2^{18} \cdot b} \\
&= \frac{\phi}{9 \cdot 2^{18}}.
\end{align*}


Now we turn to lower bounding the probability mass of $R_1$ wrt $D_X$. We first project $x$ down to the subspace spanned by $\cbr{w, u}$ - call the projected value $z = (z_1, z_2) \in \RR^2$. Observe that $z$ can also be seen as drawn from an isotropic log-concave distribution in $\RR^2$; denote by $f_Z$ its probability density function.

Without loss of generality, suppose $w = (0,1)$ and $u = (\sin\phi, \cos\phi)$. It can be now seen that $x \in R_1$ iff $z$ lies in the parallelogram $ABDC$, denoted as $\tilde{R}_1$, where $A = (\frac{1}{36} + \frac{b}{\tan\phi}, b)$, $B = (\frac{1}{18} + \frac{b}{\tan\phi}, b)$, $C = (\frac{1}{36}, 0)$, $D = (\frac{1}{18}, 0)$.
See Figure~\ref{fig:f-acute} for an illustration.
Crucially, $\| \overline{OC} \| = \| \overline{CD} \| = \frac{1}{36}$, $\| \overline{AC}\| = \| \overline{BD} \| = \frac{b}{\sin\phi} \leq \frac{1}{18}$, as $b \leq \frac{\phi}{36} \leq \frac{\sin\phi}{18}$. Therefore, by triangle inequality, all four vectices, $A, B, C, D$ have distance at most $\frac19$ to the origin. Therefore, for all $z \in \tilde{R}_1$, $\| z \| \leq \frac19$. By  Lemma~\ref{lem:ilc-density-lb}, this implies that $f_Z(z) \geq 2^{-16}$ for all $z$ in $\tilde{R}_1$. Moreover, the area of parallelogram $\tilde{R}_1$ is equal to $b \cdot \frac{1}{36} = \frac{b}{36}$.

Therefore,
\begin{equation*}
\PP_{x \sim D_X}\del{x \in R_1} = \PP_{z \sim D_Z}\del{z \in \tilde{R}_1} = \int_{\tilde{R}_1} f_Z(z) dz \geq 2^{-16} \cdot \frac{b}{36} = \frac{b}{9 \cdot 2^{18}}.
\end{equation*}
This completes the proof of the claim.

\item The proof of the second item uses similar lines of reasoning as the first. We denote by $\phi \defeq \pi - \theta(u, w)$.
Define region
\begin{equation*}
R_2 = \cbr{x: \inner{w}{x} \in [-b, 0], \inner{u}{x} \in \sbr{\frac{\sin\phi}{36}, \frac{\sin\phi}{18}}}.
\end{equation*}
It can be easily seen that $R_2$ is a subset of the disagreement region between $w$ and $u$. In other words,
\begin{equation*}
\ind{x \in R_2} \leq \ind{\sign{\inner{w}{x}} \neq \sign{\inner{u}{x}}}.
\end{equation*}

It suffices to show that, region $R_2$ has probability mass at least $\frac{b}{9 \cdot 2^{18}}$ wrt $D_X$. To see why it completes the proof, observe that
\begin{align*}
&\ \EE_{x \sim D_X} \sbr{\abs{u \cdot x} \ind{\abs{\inner{w}{x}} \leq b} \ind{\sign{\inner{w}{x}} \neq \sign{\inner{u}{x}}}} \\
\geq&\  \EE_{x \sim D_X} \sbr{\abs{u \cdot x} \ind{x \in R_2}} \\
\geq&\ \frac{\sin\phi}{36} \cdot \EE_{x \sim D_X} \ind{x \in R_2} \\
=&\ \frac{\phi}{72} \PP_{x \sim D_X}\del{x \in R_2}
\geq  \frac{b \cdot \phi}{3^4 \cdot 2^{21}},
\end{align*}
where the first inequality uses the fact that $R_2$ is a subset of both $\cbr{x: \abs{\inner{w}{x}} \leq b}$ and $\cbr{x: \sign{\inner{w}{x}} \neq \sign{\inner{u}{x}}}$; the second inequality uses the fact that for all $x$ in $R_2$, $\abs{u \cdot x} \geq \frac{\sin \phi}{36}$; the third inequality uses the elementary fact that $\sin\phi \geq \frac{\phi}{2}$.

As $\PP_{x \sim D_X} \del{ \abs{\inner{w}{x}} \leq b } \leq b$ by Lemma~\ref{lem:ilc-density-ub}, this implies that
\begin{align*}
f_{u,b}(w) = \frac{\EE_{x \sim D_X} \sbr{\abs{u \cdot x} \ind{\abs{\inner{w}{x}} \leq b} \ind{\sign{\inner{w}{x}} \neq \sign{\inner{u}{x}}}}}{\PP_{x \sim D_X} \del{ \abs{\inner{w}{x}} \leq b }}
&\geq \frac{\phi \cdot b}{3^4 \cdot 2^{21} \cdot b} \\
&= \frac{\phi}{3^4 \cdot 2^{21}}.
\end{align*}

Now we lower bound the probability mass of $R_2$ wrt $D_X$. We first project $x$ down to the subspace spanned by $\cbr{w, u}$ - call the projected value $z = (z_1, z_2) \in \RR^2$. Observe that $z$ can also be seen as drawn from an isotropic log-concave distribution on $\RR^2$; denote by its density $f_Z(z)$.

Without loss of generality, suppose $w = (0,1)$ and $u = (\sin\phi, -\cos\phi)$. It can be now seen that $x \in R_2$ iff $z$ lies in the parallelogram $CDBA $, denoted as $\tilde{R}_2$, where $A = (\frac{1}{36} - \frac{b}{\tan\phi}, -b) $, $B = (\frac{1}{18} - \frac{b}{\tan\phi}, -b)$, $C = (\frac{1}{36}, 0)$, $D = (\frac{1}{18}, 0)$. See Figure~\ref{fig:f-obtuse} for an illustration.
Crucially, $\| \overline{OC} \| = \| \overline{CD} \| = \frac{1}{36}$, $\| \overline{AC}\| = \| \overline{BD} \| = \frac{b}{\sin\phi} \leq \frac{1}{18}$, as $b \leq \frac{\phi}{36} \leq \frac{\sin\phi}{18}$.
Therefore, by triangle inequality, all four vertices $A, B, C, D$ have distance at most $\frac19$ to the origin. Therefore, for all $z \in \tilde{R}_1$, $\| z \| \leq \frac19$. This implies that $f_Z(z) \geq 2^{-16}$ for all $z$ in $\tilde{R}_2$.
Moreover, the area of parallelogram $\tilde{R}_2$ is equal to $b \cdot \frac{1}{36} = \frac{b}{36}$.

Therefore,
\begin{equation*}
\PP_{x \sim D_X}\del{x \in R_2} = \PP_{z \sim D_Z}\del{z \in \tilde{R}_2} = \int_{\tilde{R}_2} f_Z(z) dz \geq 2^{-16} \cdot \frac{b}{36} = \frac{b}{9 \cdot 2^{18}}.
\end{equation*}
\end{enumerate}
This completes the proof of the claim.
\end{proof}

\begin{figure}
\hspace{3cm}
\includegraphics[trim={0cm 4.5cm 0cm 0cm},clip,width=0.7\textwidth]{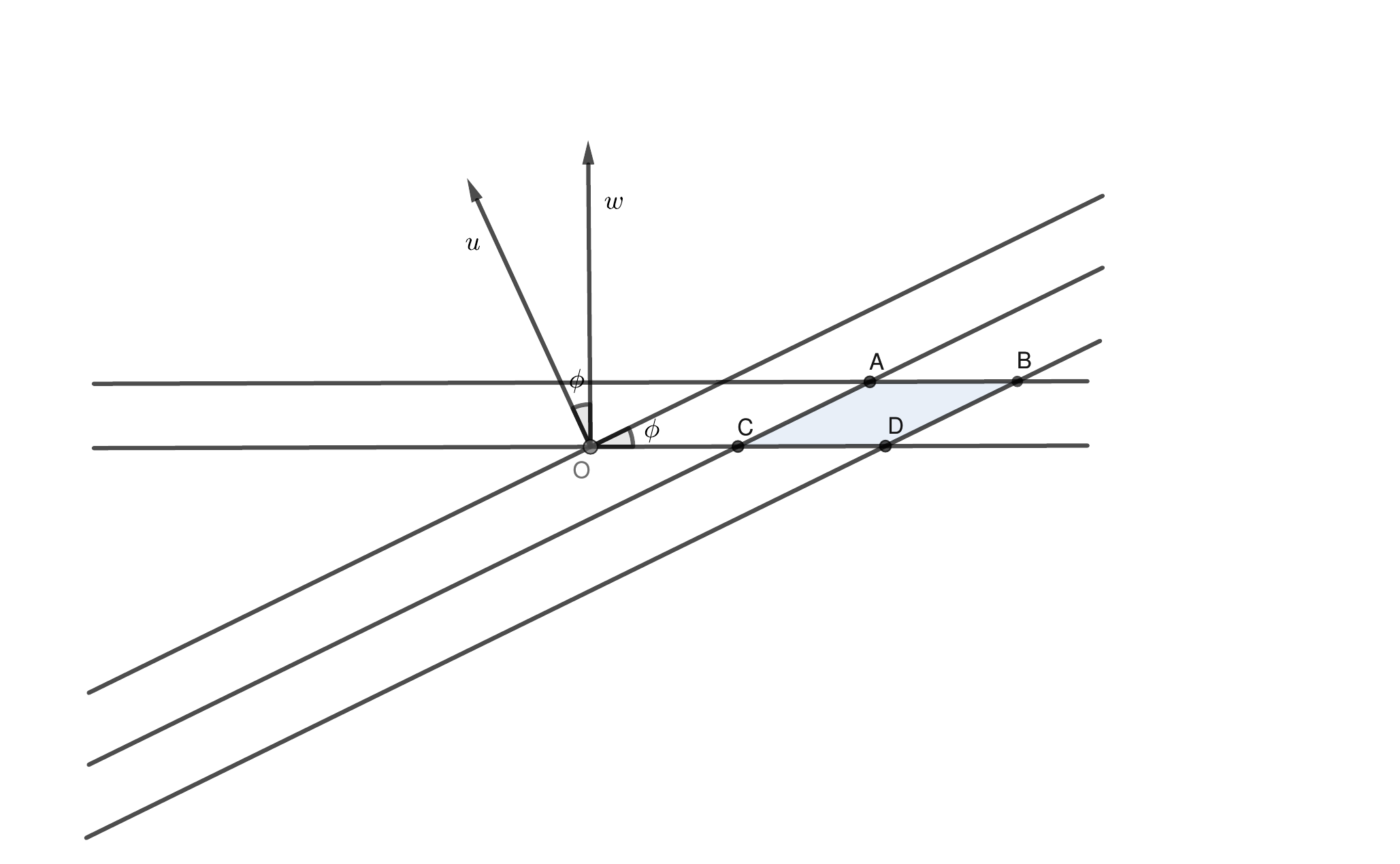}
\caption{An illustration of parallelogram region $\tilde{R}_1$ (the shaded region). Its four boundaries are: lines $AB$ and $CD$, which are $\cbr{z: \inner{w}{z} = b}$ and $\cbr{z: \inner{w}{z} = 0}$; lines $AC$ and $BD$, which are $\cbr{z: \inner{u}{z} = -\frac{\sin\phi}{36}}$ and $\cbr{z: \inner{u}{z} = -\frac{\sin\phi}{18}}$ respectively.}
\label{fig:f-acute}
\end{figure}

\begin{figure}
\includegraphics[trim={0cm 10cm 0cm 0cm},clip,width=\textwidth]{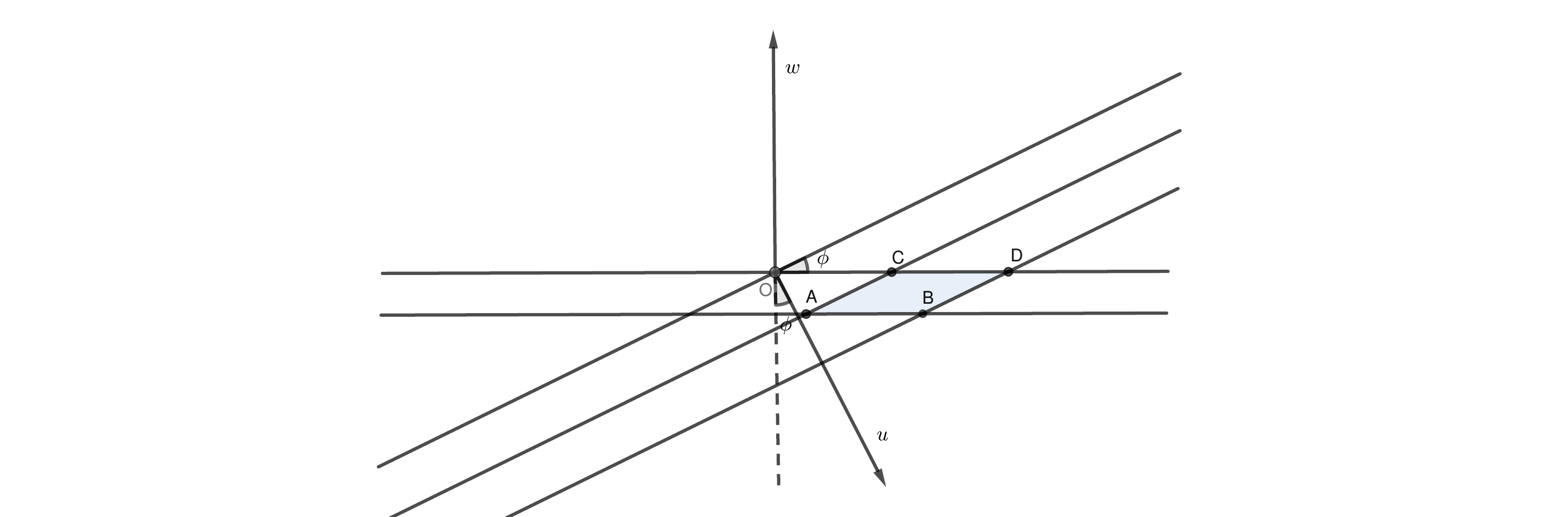}
\caption{An illustration of parallelogram region $\tilde{R}_2$ (the shaded region). Its four boundaries are: lines $AB$ and $CD$, which are $\cbr{z: \inner{w}{z} = -b}$ and $\cbr{z: \inner{w}{z} = 0}$; lines $AC$ and $BD$, which are $\cbr{z: \inner{u}{z} = \frac{\sin\phi}{36}}$ and $\cbr{z: \inner{u}{z} = \frac{\sin\phi}{18}}$ respectively.}
\label{fig:f-obtuse}
\end{figure}



\section{Basic Inequalities}

\begin{lemma}\label{lem:cos-quad}
If $\theta \in [0,\pi]$, then:
\begin{enumerate}
\item $\cos \theta \leq 1 - \frac{\theta^2}{5}$. \label{item:0-ub}
\item $\cos \theta \geq 1 - \frac{\theta^2}{2}$. \label{item:0-lb}
\item $\cos \theta \leq -1 + \frac{1}{2}(\theta - \pi)^2$. \label{item:pi-ub}
\item $\cos \theta \geq -1 + \frac{1}{5}(\theta - \pi)^2$. \label{item:pi-lb}
\end{enumerate}
\end{lemma}

\begin{lemma}[Averaging effects on angle]\label{lem:avg-angle}
Suppose we have a sequence of unit vectors $w_1, \ldots, w_T$. Let $\tilde{w} = {\frac1T \sum_{t=1}^T w_t}$ be their average.
Suppose $\frac1T \sum_{t=1}^T \cos\theta(w_t, u) \geq 0$.
Then,
$\cos\theta(\tilde{w}, u) \geq \frac1T \sum_{t=1}^T \cos\theta(w_t, u)$.
\end{lemma}
\begin{proof}
We note that
\begin{equation*}
\inner{\tilde{w}}{u} = \frac{1}{T} \sum_{t=1}^T \inner{w_t}{u}
 = \frac1T \sum_{T=1}^T \cos\theta(w_T, u) \geq 0.
\end{equation*}
In addition, by the convexity of $\ell_2$ norm, $\| \tilde{w} \|_2 = \| \frac{1}{T} \sum_{t=1}^T w_t \|_2 \leq \frac{1}{T} \sum_{t=1}^T \| w_t \| \leq 1$. This implies that
\begin{equation*}
\cos \theta(\tilde{w}, u) = \inner{\frac{\tilde{w}}{\| \tilde{w} \|}}{u}
\geq \inner{\tilde{w}}{u} = \frac1T \sum_{t=1}^T \cos\theta(w_t, u).
\end{equation*}
\end{proof}

\begin{lemma}\label{lem:infty-q}
Recall that $q = \ln(8d)$. Then for every $x$ in $\RR^d$, $\| x \|_q \leq 2 \| x \|_\infty$.
\end{lemma}
\begin{proof}
By algebra, $
\| x \|_q = \del{\sum_{i=1}^d \abs{x_i}^q }^{\frac1q} \leq (d \| x \|_\infty^q)^{\frac{1}{q}} \leq 2 \| x \|_\infty.
$
\end{proof}

We need the following elementary lemmas in our proofs. See e.g. \cite{zhang2018efficient} for the proof.

\begin{lemma}\label{lem:best-s-approx}
If $v, u$ are two vectors in $\RR^d$, and $u$ is $s$-sparse, then, $\| \Hcal_s(v) - u \|_2 \leq 2 \| v - u \|_2$.
\end{lemma}

\begin{lemma}\label{lem:normalize}
Suppose $v$ is a unit vector in $\RR^d$. Then for any $w$ in $\RR^d$, $\| \hat{w} - v \|_2 \leq 2 \| w - v \|_2$.
\end{lemma}

\begin{lemma}\label{lem:dist-angle}
If $v$ is a unit vector in $\RR^d$, and $w$ is a vector in $\RR^d$, then $\theta(w, v) \leq \pi \| w - v\|_2$.
\end{lemma}

\section{Probability Tail Bounds}
\label{sec:conc}

In this section we present a few well-known results about concentrations of random variables and martingales that are instrumental in our proofs. We include the proofs of some of the results here because we would like to explicitly track dependencies on relevant parameters.

We start by recalling a few  facts about subexponential random variables;
see e.g.~\cite{vershynin2018high} for a more thorough treatment on this topic.

\begin{definition}
A random variable $X$ with is called $(\sigma, b)$-subexponential, if
for all $\lambda \in [-\frac1b, \frac1b]$,
\begin{equation}
\EE e^{\lambda (X - \EE[X])} \leq e^{\frac{\sigma^2 \lambda^2}{2}}.
\end{equation}
\end{definition}


\begin{lemma}
Suppose $Z$ is $(\sigma, b)$-subexponential, then with probability $1-\delta$,
\begin{equation*}
\abs{Z - \EE Z} \leq \sqrt{2 \sigma^2 \ln\frac{2}{\delta}} + 2b \ln\frac{2}{\delta}.
\end{equation*}
\end{lemma}

\begin{lemma}\label{lem:subexp-tail}
Suppose $X_1, \ldots, X_n$ are iid $(\sigma, b)$-subexponential random variables, then $\frac1n \sum_{i=1}^n X_i$ is $(\frac{\sigma}{\sqrt{n}}$, $\frac{b}n)$-subexponential. Consequently, with probability $1-\delta$,
\begin{equation*}
\abs{\frac1n \sum_{i=1}^n \del{X_i - \EE \sbr{X_i}} } \leq \sqrt{\frac{2 \sigma^2}{n} \ln\frac{2}{\delta}} + \frac{2b}{n} \ln\frac{2}{\delta}.
\end{equation*}
\end{lemma}

We next show the following fact: if a random variable has a subexponential tail probability, then it is subexponential.  

\begin{lemma}\label{lem:mgf-from-tail}
Suppose $Z$ is a random variable such that $\PP( \abs{Z} > a ) \leq C \exp\del{-\frac{a}{\sigma}}$ for some $C \geq 1$. Then,
\begin{equation*}
\EE e^{\frac{\abs{Z}}{2\sigma(\ln C + 1)}} \leq 4.
\end{equation*}
\end{lemma}
\begin{proof}
We bound the left hand side as follows:
\begin{eqnarray*}
\EE e^{\frac{\abs{Z}}{2\sigma(\ln C + 1)}} &=& \int_{0}^\infty \PP \del{ e^{\frac{\abs{Z}}{2\sigma(\ln C + 1)}} \geq s} ds \\
&=& \int_0^\infty \PP\del{ \abs{Z} \geq 2\sigma(\ln C + 1) \ln s} ds \\
&\leq& \int_0^\infty \min\del{1, \frac{C}{s^{2(\ln C + 1)}}} ds \\
&=& \int_0^e \min\del{1, \frac{C}{s^{2(\ln C + 1)}}} ds + \int_e^\infty \min\del{1, \frac{C}{s^{2(\ln C + 1)}}} ds \\
&\leq& e + \int_e^\infty C e^{-2\ln C} s^{-2} ds \leq 4.
\end{eqnarray*}
where the first equality is from a basic equality for nonnegative random variable $Y$: $\EE\sbr{Y} = \int_0^\infty \PP(Y \geq t) dt$; the second equality is by rewriting the event in terms of $\abs{Z}$; the first inequality is from the assumption on $\abs{Z}$'s tail probability and the simple fact that the probability of an event is always at most $1$; the third equality is by decomposing the integration to integration on two intervals; the second inequality uses the fact that the first integral is at most $e$, and the integrand in the second integral is at most $C e^{-2\ln C} s^{-2}$ as $s \geq e$; the last inequality uses the fact that $C \geq 1$ and $e + \frac1e \leq 4$.
\end{proof}

\begin{lemma}
For random variable $Z$ and some $\lambda_0 \in \RR_+$, if $\EE \exp\del{\lambda_0 \abs{Z}} \leq C_0$, then $Z$ is
$(\frac{4\sqrt{C_0}}{\lambda_0},\frac{4}{\lambda_0})$-subexponential.
\label{lem:mgf-quad-bootstrap}
\end{lemma}
\begin{proof}
As $\EE \exp\del{\lambda_0 \abs{Z}} = \sum_{i=0}^\infty \frac{\EE \abs{Z}^i \lambda_0^i}{i!}$, where each summand is an nonnegative number,
we have that for all $i$,
\begin{equation}
\frac{\EE \abs{Z}^i \lambda_0^i}{i!} \leq \EE \exp\del{\lambda_0 \abs{Z}} \leq C_0.
\label{eqn:moment-bound}
\end{equation}
where the second inequality is by our assumption.

We introduce a new random variable $Z'$ such that $Z'$ has the exact same distribution as $Z$, and is independent of $Z$. Observe that $Z - Z'$ has a symmetric distribution, and therefore $\EE(Z-Z')^i = 0$ for all odd $i$.
We look closely at the moment generating function of $Z - Z'$:
\begin{equation*}
\EE \exp\del{\lambda(Z - Z')}
= \sum_{i=0}^\infty \frac{\EE (Z - Z')^{i}}{i!} \lambda^{2i}
= \sum_{i=0}^\infty \frac{\EE (Z - Z')^{2i}}{(2i)!} \lambda^{2i}
\end{equation*}
where the second equality uses the fact that $Z-Z'$ has a symmetric distribution.
Importantly, by the conditional Jensen's Inequality and the convexity of exponential function,
$\EE \exp\del{\lambda(Z - \EE[Z])} \leq \EE \exp\del{\lambda(Z - Z')}$.
Therefore, it suffices to bound $\EE \exp\del{\lambda(Z - Z')}$ for all $\lambda \in [-\frac {\lambda_0} 4, \frac {\lambda_0} 4]$.

We have the following sequence of inequalities:
\begin{eqnarray*}
\EE \exp\del{\lambda (Z - Z')} &=& \sum_{i=0}^\infty \frac{\EE \sbr{\abs{Z - Z'}^{2i}} \lambda^{2i}}{(2i)!} \\
&\leq& 1 + \sum_{i=1}^\infty \frac{\EE\sbr{\abs{Z}^{2i}} 2^{2i} \lambda_0^{2i}}{(2i)!} \cdot \del{\frac{\lambda}{\lambda_0}}^{2i} \\
&\leq& 1 + C_0 \sum_{i=1}^\infty \del{\frac{2\lambda}{\lambda_0}}^{2i} \\
&\leq& 1 + 2 C_0 \del{\frac{2\lambda}{\lambda_0}}^2  \\
&\leq& \exp\del{\frac{8 C_0}{\lambda_0^2} \lambda^2}.
\end{eqnarray*}
where the first inequality we separate out the first constant term, and use the basic fact that $\abs{z - z'}^{j} \leq 2^{j-1} (\abs{z}^j + \abs{z'}^j)$ for all $j \geq 1$, and the fact that $Z$ and $Z'$ has the same distribution; the second inequality uses Equation~\eqref{eqn:moment-bound} that $\frac{\EE \abs{Z}^{2i} \lambda_0^{2i}}{(2i)!} \leq C_0$; the third inequality uses
condition that $\abs{\frac{\lambda}{\lambda_0}} \leq \frac{1}{4}$, and the elementary calculation that $\sum_{i=1}^\infty \del{\frac{\lambda}{\lambda_0}}^{2i} = \del{\frac{2\lambda}{\lambda_0}}^{2} \cdot \frac{1}{1 - (\frac{2\lambda}{\lambda_0})^2} \leq 8\del{\frac{\lambda}{\lambda_0}}^{2}$; the last inequality uses the simple fact that $1+x \leq e^x$ for all $x$ in $\RR$.


To conclude, we have that for all $\lambda \in [-\frac {\lambda_0} 4, \frac {\lambda_0} 4]$,
\begin{equation*}
\EE \exp\del{\lambda (Z - \EE Z)} \leq \exp\del{\frac{8C_0}{\lambda_0^2} \lambda^2},
\end{equation*}
meaning that $Z$ is $\del{ \frac{4\sqrt{C_0}}{\lambda_0} ,\frac{4}{\lambda_0} }$-subexponential.
\end{proof}

Importantly, based on the above two lemmas we have the following subexponential property of isotropic log-concave random variables.
\begin{lemma}
If $X$ is a random variable drawn from a 1-dimensional isotropic log-concave distribution $D_X$, then $X$ is $(32, 16)$-subexponential.
Moreover, for any random variable $Y$ such that $\abs{Y} \leq 1$ almost surely, $Y X$ is also $(32, 16)$-subexponential.
\label{lem:log-concave-subexp}
\end{lemma}
\begin{proof}
By Lemma~\ref{lem:log-concave-tail}, we have that
$\PP(\abs{X} \geq t) \leq e \cdot e^{-t}$.
Applying Lemma~\ref{lem:mgf-from-tail} with $\sigma = 1$ and $C = e$, we have that $\EE e^{\frac{\abs{X}}{4}} \leq 4$. Now, using Lemma~\ref{lem:mgf-quad-bootstrap} with $\lambda_0 = \frac14$ and $C_0 = 4$, we have that $X$ is $(32, 16)$-subexponential.
The second statement follows from the exact same line of reasoning, starting from $\PP(\abs{YX} \geq t) \leq \PP(\abs{X} \geq t) \leq e \cdot e^{-t}$.
\end{proof}

In the two lemmas below, we use the shorthand that $\EE_t \sbr{\cdot} \defeq \EE\sbr{\cdot \mid \Fcal_t}$, and $\PP_t\rbr{\cdot} \defeq \PP\rbr{\cdot \mid \Fcal_t}$.

%
%
%
%


We need the following standard martingale concentration lemma (see e.g.~\cite[Theorem 2.19]{wainwright2019high}) where the conditional distribution of each martingale difference term has a subexponential distribution.
\begin{lemma}
\label{lem:almost-hoeff}
Suppose $\cbr{Z_t}_{t=1}^T$ is sequence of random variables adapted to filtration $\cbr{\calF_t}_{t=1}^m$. In addition, each random variable $Z_t$ is  conditionally $(\sigma, b)$-subexponential, formally,
\begin{equation}
\EE_{t-1}\sbr{ \exp\del{\lambda\del{ Z_t - \EE_{t-1}\sbr{Z_t} }} } \leq \exp\del{\frac{\sigma^2 \lambda^2}{2}}, \forall \lambda \in \sbr{-\frac1b, \frac1b}.
\label{eqn:cond-mgf}
\end{equation}
Then with probability $1-\delta$,
\begin{equation*}
\abs{\sum_{t=1}^T \del{ Z_t - \EE_{t-1} Z_t} } \leq \sigma \sqrt{2T \ln\frac{2}{\delta}} + 2b \ln \frac{2}{\delta}.
\end{equation*}
\end{lemma}
\begin{proof}
As all $Z_t$'s are conditionally $(\sigma, b)$-subexponential,  
Theorem 2.19 of~\cite{wainwright2019high} implies that $\sum_{t=1}^T \del{ Z_t - \EE_{t-1} Z_t}$ is $(\sigma\sqrt{T}, b)$-exponential, and for any $a > 0$,
$$
\PP\del{ \abs{\sum_{t=1}^T \del{ Z_t - \EE_{t-1} Z_t}} > a } \leq \max( 2 e^{-\frac{a^2}{2T\sigma^2}}, 2e^{-\frac{a}{2b}}  ).
$$
Taking $a_0 = \max\del{ \sqrt{2T \ln\frac 2 \delta}, 2b \ln\frac 2 \delta }$, we have $\PP\del{ \abs{\sum_{t=1}^T \del{ Z_t - \EE_{t-1} Z_t}} > a_0 } \leq \delta$.
The lemma is concluded by observing that $a_0 \leq \sqrt{2T \ln\frac 2 \delta} + 2b \ln\frac 2 \delta$.
\end{proof}

Combining Lemmas~\ref{lem:mgf-from-tail},~\ref{lem:mgf-quad-bootstrap} and~\ref{lem:almost-hoeff}, we have the following useful inequality on the concentration of a martingale where each martingale difference has a subexponential probability tail.
We note that Freedman's Inequality or Azuma-Hoeffding's Inequality does not directly apply, as they require the martingale difference to be almost surely bounded.
A similar result for subgaussian martingale differences is shown in~\cite{shamir2011variant}; see also the discussions therein.

\begin{lemma}\label{lem:subexp-tail-hoeff}
Suppose $\cbr{Z_t}_{t=1}^T$ is sequence of random variables adapted to filtration $\cbr{\calF_t}_{t=1}^T$.
For every $Z_t$, we have that $\PP_{t-1}( \abs{Z_t} > a ) \leq C \exp\del{-\frac{a}{\sigma}}$ for some $C \geq 1$.
Then, with probability $1-\delta$,
\begin{equation*}
\abs{\sum_{t=1}^T Z_t - \EE_{t-1} Z_t } \leq 16\sigma(\ln C + 1) \del{\sqrt{2 T \ln\frac{2}{\delta}} + \ln\frac{2}{\delta}}.
\end{equation*}
\end{lemma}
\begin{proof}
First, by Lemma~\ref{lem:mgf-from-tail}, we have that $\EE_{t-1} \exp\del{\frac{\abs{Z}}{2\sigma(\ln C + 1)}} \leq 4$. Therefore, using Lemma~\ref{lem:mgf-quad-bootstrap}, we have that $Z$ is $(16\sigma(\ln C + 1), 8\sigma(\ln C + 1))$-subexponential.

Therefore, by Lemma~\ref{lem:almost-hoeff}, we have that with probability $1-\delta$,
\begin{eqnarray*}
\abs{\sum_{t=1}^T Z_t - \EE_{t-1} Z_t }
&\leq& 16\sigma(\ln C + 1) \ln\frac{2}{\delta} + 16 \sigma(\ln C + 1) \sqrt{2 T \ln\frac{2}{\delta}} \\
&\leq& 16\sigma(\ln C + 1) \del{\sqrt{2 T \ln\frac{2}{\delta}} + \ln\frac{2}{\delta}}.
\end{eqnarray*}
where the second inequality is by algebra.
\end{proof}

\section{Basic Facts about Isotropic Log-concave Distributions}

The following useful lemmas are from~\cite{lovasz2007geometry}.

\begin{lemma}\label{lem:ilc-density-lb}
The statement below holds for $d = 1,2$.
Suppose $D_X$ is an isotropic log-concave distribution on $\RR^d$, with probability density function $f$. Then, for all $x$ such that $\| x \|_2 \leq \frac{1}{9}$, $f(x) \geq 2^{-16}$.
\end{lemma}
\begin{proof}
For any $d = 1,2$, by items (a) and (d) of~\cite[Theorem 5.14]{lovasz2007geometry}, we have that
for every $x$ such that $\| x \|_2 \leq \frac{1}{9}$, $f(x) \geq 2^{-9n\| x \|_2} f(0) \geq 2^{-d} f(0)$, and $f(0) \geq 2^{-7d}$.
Therefore, for $x$ such that $\| x \|_2 \leq \frac19$, $f(x) \geq 2^{-7d} \cdot 2^{-d} = 2^{-8d} \geq 2^{-16}$.
\end{proof}

\begin{lemma}\label{lem:ilc-density-ub}
If $x$ is a random variable drawn from a 1-dimensional isotropic log-concave distribution, then for all $a, b \in \RR$ such that $a < b$,
\begin{equation*}
\PP(x \in [a,b]) \leq b - a.
\end{equation*}
\end{lemma}

\begin{lemma}\label{lem:log-concave-tail}
If $x$ is a random variable drawn from a 1-dimensional isotropic log-concave distribution, then for every $t \geq 0$,
\begin{equation*}
\PP(\abs{x} > t) \leq e^{-t+1}.
\end{equation*}
\end{lemma}